\documentclass[11pt]{article}

\usepackage[final]{acl}

\usepackage{times}
\usepackage{latexsym}

\usepackage[T1]{fontenc}
\usepackage[utf8]{inputenc}

\usepackage{inconsolata}
\usepackage{hyperref}
\usepackage{graphicx}
\usepackage{tabularx}
\usepackage{enumitem}
\usepackage{url}
\usepackage{booktabs}
\usepackage{amsfonts,amsmath,amssymb}
\usepackage{nicefrac}
\usepackage{microtype}
\usepackage{multirow}
\usepackage[dvipsnames,table]{xcolor}
\usepackage[most]{tcolorbox}
\usepackage{microtype}
\usepackage{algorithm}
\usepackage{algpseudocode}
\usepackage{xspace}

\newif\ifshowdiff
\showdifftrue
\DeclareRobustCommand{\diffsent}[1]{%
  \ifshowdiff
    \begingroup\color{black}#1%
    \edef\diffsentrestore{\endgroup\spacefactor=\the\spacefactor\relax}%
    \diffsentrestore
  \else
    #1%
  \fi
}
\DeclareRobustCommand{\diffmath}[1]{%
  \ifshowdiff
    {\color{black}#1}%
  \else
    #1%
  \fi
}

\newcommand{\nfe}[1]{\textcolor{NavyBlue}{\scriptsize #1}}
\newcommand{\tps}[1]{\textcolor{ForestGreen}{\scriptsize #1}}
\newcommand{\std}[1]{\textcolor{gray}{#1}}

\DeclareRobustCommand{\modelname}{\mbox{\textbf{CD$^{4}$LM}}}

\usepackage{amsthm}

\newtheorem{proposition}{Proposition}  
          
\newtheorem{lemma}{Lemma}              
\newtheorem{definition}{Definition} 
\newtheorem{corollary}{Corollary}

\title{\modelname: Consistency Distillation and aDaptive Decoding for Diffusion Language Models}

\author{
  Yihao Liang$^{1,2}$ \quad
  Ze Wang$^{2}$ \quad
  Ximeng Sun$^{2}$ \quad
  Jiang Liu$^{2}$ \quad
  Jialian Wu$^{2}$
  \\[1.5ex]
  \bfseries Hao Chen$^{2}$ \quad
  Xiaodong Yu$^{2}$ \quad
  Emad Barsoum$^{2}$ \quad
  Zicheng Liu$^{2}$ \quad
  Niraj K. Jha$^{1}$
  \\[1.5ex]
  \normalfont $^{1}$ Princeton University \qquad
  $^{2}$ AMD \\[2pt]
  \texttt{\{yhliang,jha\}@princeton.edu}
}

\begin{document}
\maketitle

\begin{abstract}
Diffusion language models (DLMs) accelerate generation by decoding multiple tokens in parallel, but pushing them to \textbf{fast decoding often sacrifices accuracy}. We attribute this to a static-to-dynamic mismatch: training optimizes local denoising under fixed schedules, while fast inference skips steps and encounters irregular mask patterns. Existing accelerators largely address only one side of this problem. Training-free inference methods speed up decoding without changing the model, leaving the mismatch unresolved; training-side distillation methods adapt the model, but often tie the student to a specific decoding trajectory, limiting composability and still suffering accuracy loss.
We propose \modelname{}, a composable framework for accurate DLM acceleration. Its core training component, \textit{Discrete-Space Consistency Distillation (DSCD)} trains a trajectory-invariant student on diverse masked states, reducing the mismatch without fixing a decoding policy. We pair DSCD with a lightweight \textit{Confidence-Adaptive Decoder}, while keeping the trained student compatible with external inference-side accelerators.
On GSM8K, \modelname{} preserves accuracy while achieving a 5.18$\times$ speedup; stacking \modelname{} with per-step cache accelerators further pushes speedup to 9.73$\times$ at teacher-level accuracy, suggesting a scalable path for combining training-side robustness with future decoding advances.
Code is available at \url{https://github.com/yihao-liang/CDLM}.

\end{abstract}

\section{Introduction and Related Work}

Diffusion language models (DLMs) such as LLaDA~\citep{nie2025large} and Dream~\citep{ye2025dream} offer a compelling alternative to autoregressive (AR) language models~\citep{achiam2023gpt,grattafiori2024llama} by generating text through iterative denoising rather than strict left-to-right prediction. Because each denoising step can update multiple masked tokens in parallel, DLMs promise to replace the sequential dependency of AR decoding with a smaller number of parallel refinement steps. This promise is especially attractive for structured generation tasks such as mathematical reasoning and code synthesis~\citep{hendrycks2021measuring,austin2021program,chen2021evaluating}, where long outputs and multi-step derivations make decoding latency a practical bottleneck.

In practice, the parallelism promised by DLMs only materializes when the number of function evaluations (NFE) is aggressively reduced, and pushing DLMs to low-NFE decoding causes substantial accuracy degradation. We attribute this to a \emph{static-to-dynamic mismatch}: training optimizes local denoising under fixed corruption schedules, while fast inference skips intermediate states and commits multiple tokens at once, inducing irregular mask patterns far from those seen during training. As a result, the model is then asked to refine off-schedule masked configurations for which its conditional predictions and confidence estimates are poorly calibrated.

Existing DLM accelerators tackle this from two largely separate directions, but neither simultaneously delivers speed, accuracy, and composability. Training-free inference accelerators---including key-value (KV) caching~\citep{liu2025dllm,jiang2025d}, confidence thresholding~\citep{wu2025fast,sahoo2024simple} and parallel unmasking heuristics~\citep{israel2025accelerating,christopher2025speculative,wei2025accelerating,wang2025d2f}---are plug-and-play across backbones, but the mismatch remains because the base model is unchanged, capping attainable speedup whenever accuracy must be preserved. Training-side distillation methods adapt the model itself for faster generation, but they typically train the student along a particular trajectory~\citep{deschenaux2024beyond,zhang2026t3d}, certainty schedule~\citep{chen2025dparallel}, attention pattern~\citep{kim2026consistency}, rollout order~\citep{qian2026d3llm}, or decoding policy. This creates a trajectory-locked student that becomes difficult to combine with inference-side accelerators inducing different intermediate states, and a noticeable accuracy gap to the teacher typically remains.

We propose that DLM acceleration should be designed around composability: a training-side accelerator should improve robustness to fast decoding without prescribing how decoding must be performed, turning training-side acceleration into a reusable layer that the inference loop can choose, adapt, or stack.

Following this principle, we introduce \modelname{}, a composable framework for accurate DLM acceleration. \diffsent{Its training component, Discrete-Space Consistency Distillation (DSCD), trains a student across diverse masked states via teacher-subset masking without fixing a decoding trajectory.} The student observes a more heavily masked view, while the teacher conditions on a lighter masked view that contains a superset of the student's visible information. \diffsent{The student is trained to match the teacher's conditional distribution while remaining anchored to a ground-truth reconstruction objective.} This exposes the student to a broad family of off-schedule mask patterns and reduces the mismatch without binding the model to any particular decoding trajectory.

We pair DSCD with a lightweight Confidence-Adaptive Decoder (CAD), which uses the student's confidence to decide how many tokens to commit at each denoising step: High-confidence tokens are finalized early, uncertain tokens remain masked for further refinement, which converts fixed schedules into a dynamic risk-efficiency trade-off. Because DSCD preserves full bidirectional attention and hard-codes no rollout sequence, the same student also drops into external inference-side accelerators such as KV caching and parallel decoding.

Empirically, \modelname{} reaches 77.6\% on GSM8K at a 5.18$\times$ wall-clock speedup (Table~\ref{tab:gsm8k}). \diffsent{Stacking the same student on top of per-step cache accelerators further pushes the frontier: \modelname{}+dLLM-Cache~\citep{liu2025dllm} reaches higher speed and \modelname{}+d$^{2}$Cache~\citep{jiang2025d} reaches higher accuracy, with both composed configurations Pareto-dominating the standalone accelerators reproduced in the same harness in Table~\ref{tab:gsm8k}.} A single \modelname{} student also supports both block-diffusion and pure-diffusion decoding (Appendix~\ref{sec:appendix-gsm8k-extended}). We also report transfer to DiffuCoder~\citep{gong2025diffucoder} in Appendix~\ref{sec:appendix-diffucoder}.

Our contributions are as follows.

\begin{enumerate}
    \item \textit{A composability-first formulation.} We frame DLM acceleration around the static-to-dynamic mismatch and argue that training-side acceleration should serve as a composable robustness layer rather than a replacement for inference-side acceleration.

    \item \textit{A decoding-pattern-agnostic training and adaptive decoding framework.} We propose \modelname{} = DSCD+CAD. DSCD trains a student over diverse teacher-subset masked states without committing to a fixed decoding path; CAD uses confidence-guided commitment to allocate denoising compute adaptively. \diffsent{We further give an idealized conditional-projection view of DSCD, together with a posterior-risk analysis under a Bayes-teacher assumption.}

    \item \textit{Strong accuracy-latency trade-offs and empirical composability.} \modelname{} reaches 77.6\% on GSM8K at 5.18$\times$ speedup. \diffsent{Stacking with reproduced per-step cache accelerators further pushes speedup to \textbf{9.73$\times$} (\modelname{}+dLLM-Cache, 77.6\%) and accuracy to \textbf{78.0\%} (\modelname{}+d$^{2}$Cache, 8.96$\times$), with both composed configurations Pareto-dominating the standalone baselines reproduced in the same harness in Table~\ref{tab:gsm8k}.} A single distillation pass also supports both block-diffusion and pure-diffusion decoding (Appendix~\ref{sec:appendix-gsm8k-extended}). Transfer to DiffuCoder is reported in Appendix~\ref{sec:appendix-diffucoder}.
\end{enumerate}


\section{Method}
\label{sec:method}

\begin{figure*}[t]
  \centering
  \includegraphics[width=\textwidth]{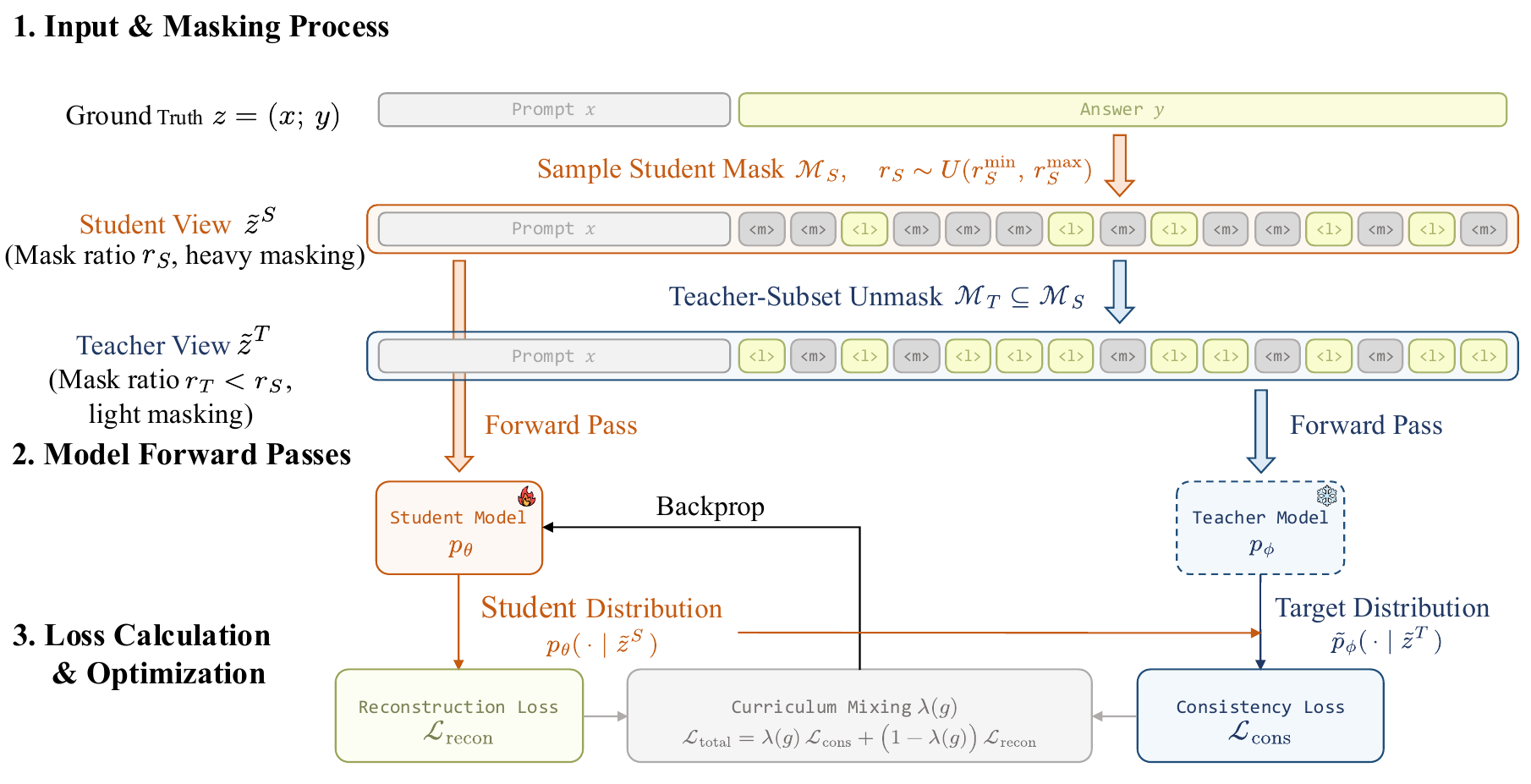}
  \caption{\textbf{Overview of the DSCD pipeline.} \diffsent{Teacher-subset masking ($\mathcal{M}_T \!\subseteq\! \mathcal{M}_S$) lets the teacher condition on a superset of the student's context, which is typically strict before integer rounding.} The student is trained jointly under $\mathcal{L}_{\text{recon}}$ (ground-truth anchor) and $\mathcal{L}_{\text{cons}}$ (alignment with the teacher's soft targets), with curriculum weight $\lambda(g)$ governing the transition from distillation to refinement.}
  \label{fig:dscd_pipeline}
\end{figure*}

\modelname{} is designed to resolve the structural misalignment between fixed diffusion training schedules and the need for flexible, budget-aware inference. 
To bridge this gap, we propose a framework built upon an absorbing-state discrete diffusion backbone (formal definitions and full notation provided in Appendix~\ref{sec:appendix-preliminaries}).
Our approach consists of two core components: 
A DSCD scheme (Sect.~\ref{sec:method:dscd}) that trains a student to be approximately trajectory-invariant along the teacher’s masking process, and a CAD policy (Sect.~\ref{sec:method:cad}) that dynamically allocates NFE across tokens and instances. Taken together, these modules decouple the trained model from any fixed diffusion schedule and enable flexible low-NFE decoding.
Furthermore, because no specific decoding pattern is embedded in the student, the resulting model is directly composable with inference-side accelerators (Sect.~\ref{sec:composability}).

\subsection{Discrete-Space Consistency Distillation}
\label{sec:method:dscd}

Given a training example consisting of a prompt $x$ and a target response $y$ (concatenated as $z = [x; y]$), our goal is to train a student model $p_\theta$ to predict the clean $z$ from a corrupted state $\tilde{z}$ in a few steps, without being tied to any particular diffusion schedule. Unlike the teacher $p_\phi$ which follows a fixed denoising schedule, we train $p_\theta$ to be approximately \emph{trajectory-invariant} along the absorbing diffusion process of the teacher.

\paragraph{Ideal trajectory invariance.}
For a clean sequence $z$ and its absorbing diffusion trajectory $\{\tilde{z}_t\}_{t\in[0,1]}$, generated by the absorbing masking process (defined in Appendix~\ref{sec:appendix-preliminaries}), an ideal consistency-style student would satisfy
\begin{equation}
\begin{aligned}
p_\theta(z \mid \tilde{z}_t)
&\;\approx\;
p_\theta(z \mid \tilde{z}_{t'})
\;\approx\;
p_\phi(z \mid \tilde{z}_t, t)\\
&\quad\text{for all } t,t'\in[0,1].
\end{aligned}
\label{eq:ideal-consistency}
\end{equation}
Unlike the teacher $p_\phi$, which conditions on explicit time $t$, the student $p_\theta$ learns to infer the effective noise level directly from the masked input $\tilde{z}$, enabling time-agnostic inference.
This requirement is too strong to enforce directly: It couples all time points on the trajectory and requires matching full sequence distributions. In practice, DSCD implements a weaker but tractable surrogate that enforces \emph{pairwise} consistency between stochastic views of the same $z$ and anchors them to the data distribution.

\paragraph{Paired teacher-subset masking.}
\diffsent{For each training example, we first sample a student mask ratio $r_S \sim \mathcal{U}(r_S^{\min}, r_S^{\max})$ and then set a lighter teacher ratio $r_T = \mathrm{clip}(r_S \cdot u,r_T^{\min},r_T^{\max})$, where $u \sim \mathcal{U}(u_{\min}, u_{\max})$ with $u_{\max}<1$.}
\diffsent{We convert each ratio to its diffusion timestep using the same noise schedule, yielding $t_S$ and $t_T$. Let $n_S=\max(1,\lfloor L_y r_S \rfloor)$ and $n_T=\min(\lfloor L_y r_T \rfloor,n_S)$. We form the student masked set $\mathcal{M}_S$ by uniformly sampling (without replacement) $n_S$ target positions from $\{L_x+1,\dots,L\}$. To construct the teacher mask, we then uniformly sample a subset $\mathcal{M}_T \subseteq \mathcal{M}_S$ of size $n_T$. This nested masking makes the teacher informationally at least as rich as the student; the inclusion is typically strict but can become equality after integer rounding. Hence, the consistency loss is evaluated on $\mathcal{M}_S$.}
The overall training pipeline is depicted in Fig.~\ref{fig:dscd_pipeline}.

\diffsent{At a high level, DSCD then combines a reconstruction loss and a Kullback-Leibler (KL) divergence-based consistency loss computed on $\mathcal{M}_S$ to (i) anchor the student to the data distribution at its own masked positions and (ii) align the student with the teacher’s predictions under a nested corruption pattern, yielding a practical relaxation of the trajectory-invariance property in \eqref{eq:ideal-consistency}.} \diffsent{This constraint ensures the teacher always conditions on a superset of the student's information, providing an information-nested distillation target (idealized posterior-risk analysis provided in Appendix~\ref{app:variance}).}

\paragraph{Training Objective.}
We minimize a joint objective $\mathcal{L}_{\text{total}} = \lambda(g)\mathcal{L}_{\text{cons}} + (1-\lambda(g))\mathcal{L}_{\text{recon}}$, where $\lambda(g)$ follows a cosine curriculum to transition from distillation to refinement.
$\mathcal{L}_{\text{recon}}$ anchors the student to the ground truth, while $\mathcal{L}_{\text{cons}}$ aligns the student with the temperature-scaled teacher distribution $\tilde{p}_\phi(\cdot|\tilde{z}^T)_i \propto \exp(l_\phi^{(i)}/\tau)$ using standard $\tau^2$ scaling~\citep{hinton2015distilling}:

\begin{equation}
    \begin{aligned}
    \mathcal{L}_{\text{recon}}(\theta)
    &=
    \mathbb{E}_{(x,y)\sim\mathcal{D}}
    \mathbb{E}_{r_S,\mathcal{M}_S}\\
    &\quad\bigg[
    \frac{1}{|\mathcal{M}_S|}
    \sum_{i\in\mathcal{M}_S}
    -\log p_\theta(z_i \mid \tilde{z}^S)
    \bigg].
    \end{aligned}
    \label{eq:recon-loss}
\end{equation}

\begin{equation}
    \begin{aligned}
    &\mathcal{L}_{\text{cons}}(\theta)
    =
    \tau^2 \cdot
    \mathbb{E}_{(x,y)\sim\mathcal{D}}
    \mathbb{E}_{r_S,r_T,\mathcal{M}_S,\mathcal{M}_T}\\
    &\quad\bigg[
    \frac{1}{|\mathcal{M}_S|}
    \sum_{i\in\mathcal{M}_S}
    \mathrm{KL}\big(
    \tilde{p}_\phi(\cdot\mid\tilde{z}^T)_i
    \big\|
    p_\theta(\cdot\mid\tilde{z}^S)_i
    \big)
    \bigg].
    \end{aligned}
    \label{eq:cons-loss}
\end{equation}

\diffsent{For $i\in\mathcal{M}_S\setminus\mathcal{M}_T$, the teacher observes $z_i$ whereas the student does not; thus, the KL term combines soft distillation on $\mathcal{M}_T$ with information-enriched reconstruction targets on $\mathcal{M}_S\setminus\mathcal{M}_T$.}
\diffsent{This formulation encourages refinement across the sampled nested mask pairs (via $\mathcal{L}_{\text{cons}}$) while retaining generation quality (via $\mathcal{L}_{\text{recon}}$): see Appendix~\ref{sec:appendix-dscd-theory} for theoretical analysis.} For a complete algorithmic description of the training procedure, including the details of mask sampling and the curriculum schedule, readers can refer to the pseudocode provided in Appendix~\ref{sec:appendix-distillation}.

\paragraph{Relation to trajectory-locked distillation.}
Concurrent distillation methods also constrain the student to a specific decoding pattern: dParallel~\citep{chen2025dparallel} aligns the student along the teacher's certainty schedule, d3LLM~\citep{qian2026d3llm} distills the teacher's pseudo-decoding order, and T3D~\citep{zhang2026t3d} applies reverse-KL distillation to teacher rollouts.
\diffsent{In contrast, DSCD supervises the student on random teacher-subset masks (Eq.~\ref{eq:cons-loss}), aligning the student with teacher targets over sampled student information sets (Appendix~\ref{sec:appendix-dscd-theory}).}
This formulation makes the student robust to downstream decoding policies that induce off-schedule mask patterns, including the cache-based policies evaluated in Sect.~\ref{sec:composability}.

\begin{figure*}[t]
  \centering
  \includegraphics[width=\textwidth]{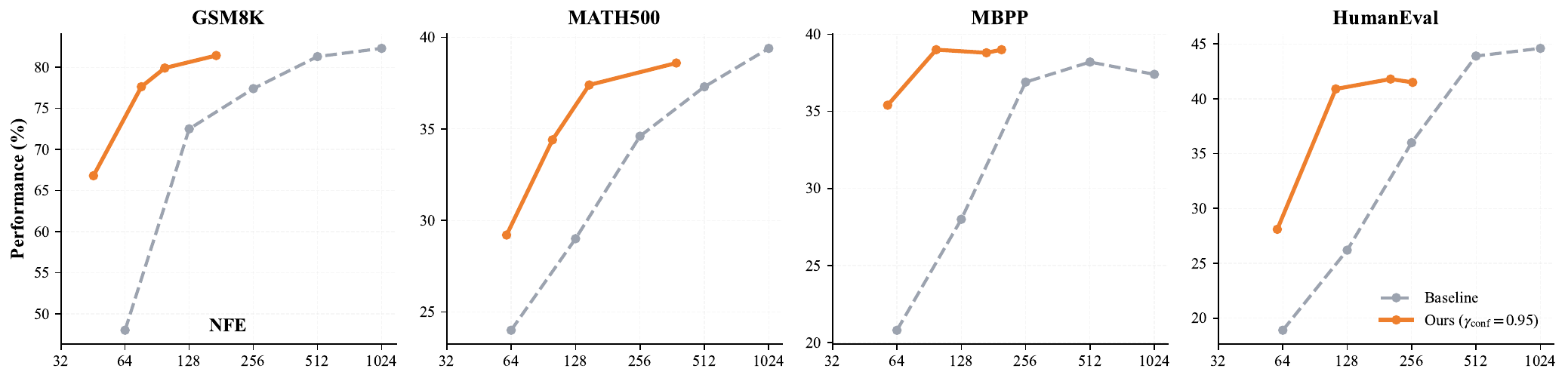}
    \caption{\textbf{Accuracy--compute trade-off under different decoding budgets.} We plot performance (\%) versus the average NFE across benchmarks. The dashed gray curve denotes the LLaDA-8B-Instruct baseline evaluated under increasing NFE budgets, while the solid orange curve denotes our CAD with the given threshold ($\gamma_{\text{conf}}=0.95$). Each marker corresponds to one decoding budget. Across all four benchmarks, our method yields a consistently better trade-off, achieving comparable or higher performance at substantially lower NFE.}

  \label{fig:bench}
\end{figure*}

\subsection{Confidence-Adaptive Decoding}
\label{sec:method:cad}

After distillation, we discard the teacher and decode using the student model only. Our CAD is a generic unmasking policy. In this work, we instantiate it within the LLaDA block diffusion framework. Pure diffusion is recovered as a special case by setting $b=L_{\text{gen}}$. Fig.~\ref{fig:decoding-overview} in Appendix~\ref{sec:appendix-decoding} contrasts CAD against fixed-schedule block diffusion, pure diffusion, and AR decoding. The same appendix gives a step-by-step illustration of a single CAD step.

\paragraph{Block diffusion schedule.}
Unlike autoregressive models that generate variable-length sequences by dynamically appending tokens, diffusion-based decoding operates on a pre-allocated, fixed-size canvas.
We set a maximum target horizon $L_{\text{gen}}$ (e.g., matching the training sequence length or a system budget) and pad the initial state with masks.
The actual sequence length is determined dynamically during decoding via the end-of-sequence (EOS) blocking mechanism.
Given $L_{\text{gen}}$ and block size $b$, we partition positions into contiguous blocks
$\{\mathcal{B}_j\}_{j=1}^{J}$, where
$\mathcal{B}_j=\{(j\!-\!1)b+1,\ldots,\min(jb,L_{\text{gen}})\}$ and $J=\lceil L_{\text{gen}}/b\rceil$.
Decoding proceeds left-to-right with an active block index $j$. Tokens from finished blocks are cached (frozen), while future blocks remain masked until activated.
Note that while our primary experiments use full-attention recomputation to isolate the algorithmic gains of DSCD, this block-wise formulation is structurally compatible with approximate KV caching schemes~\citep{wu2025fast}.

\paragraph{State, masked set, and eligible set.}
Let $\tilde{z}^{(s)}\in\mathcal{V}^{L_{\text{gen}}}$ denote the partially-masked state at decoding step $s$ and let $\mathcal{M}^{(s)}=\{i\in[L_{\text{gen}}]:\tilde{z}^{(s)}_i=\mathtt{m}\}$ be the masked set. CAD operates on the \emph{eligible set} within the active block:
\begin{equation}
\mathcal{E}^{(s)}=\mathcal{M}^{(s)}\cap \mathcal{B}_j,
\label{eq:eligible-set}
\end{equation}
\diffsent{This reduces to $\mathcal{E}^{(s)}=\mathcal{M}^{(s)}$ when $b=L_{\text{gen}}$ (pure diffusion).}

\paragraph{Confidence-adaptive set selection.}
We treat the decoding process as a dynamic risk-efficiency trade-off. The full decision cycle is illustrated in Appendix~\ref{sec:appendix-decoding} (Fig.~\ref{fig:CAD}).
Starting from a partially masked state (Step 1), the student model predicts distributions for all masked positions in a single forward pass (Step 2).
We then compute a confidence score $c_i^{(s)}$ for each eligible position $i\in\mathcal{E}^{(s)}$ (Step 3), using the maximum probability $c_i^{(s)}=\max_v p_\theta(v\mid\tilde{z}^{(s)})_i$ as a proxy for correctness.
Based on these scores, we identify a candidate set $\widehat{\mathcal{S}}^{(s)}=\{i\in\mathcal{E}^{(s)}:c_i^{(s)}\ge \gamma_{\text{conf}}\}$ containing tokens where the model's certainty outweighs the risk of error (Step 4).
To stabilize the dynamic trajectory, we clamp the commit size via
\begin{equation}
\begin{aligned}
\diffmath{k^{(s)}}&=\diffmath{\min\!\big\{|\mathcal{E}^{(s)}|,k_{\max},\max\big(|\widehat{\mathcal{S}}^{(s)}|,k_{\min}\big)\big\}},\\
\diffmath{\mathcal{S}^{(s)}}&=\diffmath{\mathrm{TopK}\big(\{c_i^{(s)}\}_{i\in\mathcal{E}^{(s)}},\;k^{(s)}\big)},
\end{aligned}
\label{eq:cad-topk-main}
\end{equation}
and finally unmask the predictions in $\mathcal{S}^{(s)}$ while keeping lower-confidence tokens masked for future refinement (Step 5). This greedy policy approximates the optimal stopping rule for per-step decoding risk (see derivation in Appendix~\ref{sec:appendix-cad-theory}).

To prevent degenerate short outputs, we also apply an EOS blocking mechanism (implementation details in Appendix~\ref{sec:appendix-eos}). Our controller introduces negligible overhead compared to the model forward pass (complexity analysis in Appendix~\ref{sec:appendix-complexity}). For the complete pseudocode of the CAD algorithm, which details the interaction between block-wise diffusion and the dynamic acceptance policy, readers can refer to Appendix~\ref{sec:appendix-decoding}.

\section{Experimental Setup}
\label{sec:experiments}

We discuss the experimental setup next.

\subsection{Models and Training}
\paragraph{Teacher and student.}
We use LLaDA-8B-Instruct~\citep{nie2025large} as the teacher. The student is initialized from the teacher's pretrained weights. All experiments use BF16 precision and a maximum context length of 1024.

\paragraph{Training data.}
We train our student model on a 200k-sample subset of OpenCodeInstruct~\citep{ahmad2025opencodeinstruct}, a comprehensive instruction-following code generation dataset. To ensure high-quality distillation, we filter for samples with valid solutions and reserve 5\% (10k samples) for validation. \diffsent{For mathematical reasoning experiments, we use the GSM8K~\citep{cobbe2021training} training split with the same DSCD objective and the task-specific hyperparameters listed in Table~\ref{tab:hparam-train-full}.}

\subsection{Benchmarks}
\paragraph{Code generation.} 
We evaluate the model in a zero-shot setting on HumanEval~\citep{chen2021evaluating} and three-shot on MBPP~\citep{austin2021program}, and also report results on the stricter variants: HumanEval+ and MBPP+. We report functional correctness using pass@1 and pass@5.

\paragraph{Mathematical reasoning.} 
We report accuracy on GSM8K~\citep{cobbe2021training} and MATH500~\citep{hendrycks2021measuring} in a zero-shot setting.

\paragraph{Evaluation framework.} 
All evaluations are implemented based on the open-sourced DAEDAL's evaluation codebase~\citep{li2025beyond}, which standardizes prompting, post-processing, and metric computation for LLaDA-style DLMs.

\subsection{Training Protocol}

\diffsent{We use a global batch size of 64 on 8$\times$AMD MI250 GPUs, AdamW at $5\times10^{-6}$, cosine decay, and 10\% learning-rate warmup (300 code steps; 200 mathematics steps). Code and mathematics run for three and ten epochs with $\tau=1.5$ and $2.0$, respectively.} Detailed hyperparameters and prompt templates are provided in Appendix~\ref{sec:appendix-hyperparameters}.

\begin{table*}[t]
  \centering
  \small
  \setlength{\tabcolsep}{6pt}
  \renewcommand{\arraystretch}{1.05}
    \begin{tabular}{lccc}
      \toprule
      \textbf{Method} &
      \textbf{Acc. (\%) $\uparrow$} & \textbf{Avg. NFE $\downarrow$} & \textbf{Tokens/s (Speedup) $\uparrow$} \\
      \midrule
      \multicolumn{4}{l}{\emph{LLaDA-8B-Instruct baselines (fixed tokens-per-step)}} \\
      LLaDA \tps{tps=1}    & 77.4 & 256.0 & 7.5 (1.00$\times$) \\
      LLaDA \tps{tps=2}    & 74.8 & 128.0 & 14.9 (1.99$\times$) \\
      LLaDA \tps{tps=4}    & 43.6 & 64.0  & 29.7 (3.96$\times$) \\
      LLaDA \tps{tps=8}    & 21.2 & 32.0  & 59.5 (7.93$\times$) \\
      \midrule
      \multicolumn{4}{l}{\emph{Training-based accelerators}$^{\dagger}$} \\
      dParallel~\citep{chen2025dparallel}            & 76.7 & 30.0 & 64.0 (8.53$\times$) \\
      CDLM~\citep{kim2026consistency}                & 73.9 & 57.7 & 33.3 (4.44$\times$) \\
      d$^{3}$LLM~\citep{qian2026d3llm}               & 74.2 & 36.8 & 52.2 (6.96$\times$) \\
      \midrule
      \multicolumn{4}{l}{\emph{Inference-based accelerators}} \\
      Fast-dLLM~\citep{wu2025fast}                   & 75.8 & 84.1 & 32.4 (4.30$\times$) \\
      d$^{2}$Cache~\citep{jiang2025d}             & 77.4 & 256 & 29.4 (3.90$\times$) \\
      dLLM-Cache~\citep{liu2025dllm}                 & 76.9 & 256 & 31.7 (4.25$\times$) \\
      \midrule
      \multicolumn{4}{l}{\emph{Ours (training + inference)}} \\
      \textbf{DSCD + CAD (Our \modelname{})}                   & 77.6 & 76.3 & \textbf{38.7 (5.18$\times$)} \\
      \textbf{DSCD + Fast-dLLM}                      & 77.2 & \textbf{70.5} & 42.0 (5.60$\times$) \\
      \textbf{Ours + d$^{2}$Cache}           & \textbf{78.0} & 76.5 & 67.0 (8.96$\times$) \\
      \textbf{Ours + dLLM-Cache}             & 77.6 & 76.6 & \textbf{72.7 (9.73$\times$)} \\
      \bottomrule
    \end{tabular}
    \caption{\textbf{GSM8K (zero-shot) accuracy-efficiency across acceleration paradigms} ($L_{\text{gen}}=256$, block size $b=32$). Methods are grouped by acceleration paradigm. The \tps{tps} next to a method denotes the preset tokens per step; methods without a \tps{} annotation use adaptive or method-specific scheduling. Speedup is wall-clock relative to the LLaDA \tps{tps=1} baseline. Methods marked $^{\dagger}$ are taken directly from the original papers under their reported settings; their tokens/s and speedup are extrapolated from the reported NFE relative to the LLaDA baseline. The remaining baselines are reproduced under the same harness as our method. Sequential ($b=1$) and pure-diffusion ($b=L_{\text{gen}}$) variants of LLaDA and our method are reported in Appendix~\ref{sec:appendix-gsm8k-extended}.}
  \label{tab:gsm8k}
\end{table*}

\subsection{Inference Protocols}
\paragraph{Generation paradigms.}
We consider three generation paradigms:
(i) Sequential ($b=1$),
(ii) Block Diffusion (block size $b=32$),
and (iii) Pure Diffusion ($b=L_{\text{gen}}$).

For mathematical reasoning benchmarks, we use block diffusion with $b=32$, which is the best-performing configuration in our tuning process and is used throughout the main mathematical experiments. We use pure diffusion for code generation. Unless specified otherwise, we use a maximum generation length of $L_{\text{gen}}=256$.

\paragraph{Sampling for pass@k.}
We compute pass@1 using greedy $\tau_{\text{samp}}=0$ decoding and pass@5 using sampling with temperature $\tau_{\text{samp}}=1.0$ (other decoding knobs follow the DAEDAL defaults for each benchmark).

\paragraph{Speed measurement.}
We report wall-clock speedup under the same hardware and evaluation harness, measured relative to the corresponding LLaDA-8B-Instruct baseline configuration for each benchmark with a batch size of 1.

\paragraph{Inference-side baselines.}
For composability analysis (Sect.~\ref{sec:composability}), we additionally reproduce Fast-dLLM~\citep{wu2025fast}, d$^{2}$Cache~\citep{jiang2025d}, and dLLM-Cache~\citep{liu2025dllm} under the same harness and default settings, so that the comparison baselines and the DSCD-composed configurations differ only in the inference-side accelerator.

\section{Results and Analysis}
\label{sec:results}

\begin{table*}[t]
  \centering
  \begin{tabular}{llccc}
    \hline
    \textbf{Task} & \textbf{Benchmark} & \textbf{LLaDA-8B-Ins} & \textbf{Ours} & \textbf{Speedup} \\
    \hline
    \multirow{6}{*}{\textbf{Code}}
      & HumanEval (pass@1 \%) & 38.7 \nfe{256} & \textbf{40.9} \nfe{113.2} & 3.30$\times$ \\
      & HumanEval (pass@5 \%) & 51.2 \nfe{256} & \textbf{52.4}\nfe{113.9} & 3.26$\times$ \\
      & HumanEval-plus (pass@1 \%) & 31.7 \nfe{256} & \textbf{32.9} \nfe{115.0} & 3.25$\times$ \\
      & HumanEval-plus (pass@5 \%) & 42.7 \nfe{256} & \textbf{43.9} \nfe{113.3} & 3.27$\times$ \\
      & MBPP-plus (pass@1 \%) & \textbf{48.7} \nfe{256} & 47.9 \nfe{108.0} & 3.51$\times$ \\
      & MBPP-plus (pass@5 \%) & \textbf{68.8} \nfe{256} & 67.7 \nfe{108.8} & 3.41$\times$ \\
    \hline
    \multirow{1}{*}{\textbf{Math}}
      & MATH500 (Acc \%) & 37.3 \nfe{512} & \textbf{38.6} \nfe{148.3} & 5.33$\times$  \\
    \hline
    \multicolumn{2}{l}{\textbf{Avg.}} 
        & 45.6 \nfe{292.6}  & \textbf{46.3} \nfe{117.2} & 3.62$\times$  \\
    \hline
  \end{tabular}
  
 \caption{\textbf{Overall performance on code and mathematics benchmarks.} We report pass@1/5 on code benchmarks and accuracy on MATH500. The \nfe{nfe} next to each score denotes the Avg. NFE for that setting. Speedup is measured by wall-clock time relative to LLaDA-8B-Instruct under the same evaluation protocol. \textbf{Avg.} denotes the unweighted mean across the listed benchmarks.}

  \label{tab:benchmark}
\end{table*}

Next, we present our experimental results and their analysis.

\subsection{Main Results}
\label{sec:main_results}

We examine how four families of DLM acceleration trade off accuracy against speed on GSM8K (Table~\ref{tab:gsm8k}), then verify that the headline configuration transfers across the broader code and mathematics suite (Table~\ref{tab:benchmark}).

\paragraph{Acceleration paradigms on GSM8K.}
\label{sec:composability}
Table~\ref{tab:gsm8k} groups accelerators into four paradigms, and the comparison exposes a consistent pattern.

(i) Naive step reduction collapses accuracy.
Increasing tokens-per-step on the unmodified teacher cuts NFE proportionally but accuracy degrades from 77.4\% at \tps{tps=1} to 74.8\% at \tps{tps=2} and finally 21.2\% at \tps{tps=8}: with the base model unchanged, speedup from the teacher alone is bought at near-total accuracy loss.

(ii) Training-based methods reduce but do not eliminate the accuracy gap.
dParallel, CDLM, and d$^{3}$LLM are designed around specific decoding patterns and recover most of the accuracy lost under naive reduction, but their reported accuracies remain 0.7-3.5 points below the teacher (76.7\%, 73.9\%, 74.2\%).

(iii) Reproduced inference-side accelerators remain below the combined frontier.
Cache-based accelerators (d$^{2}$Cache, dLLM-Cache) and the parallel-decoding-plus-cache method Fast-dLLM avoid retraining and achieve 3.9-4.3$\times$ speedup in our harness; Fast-dLLM reaches its highest reported speedup in this group at 75.8\% accuracy.

(iv) DSCD+CAD preserves accuracy and composes cleanly with caching.
\modelname{} reduces NFE to 76.3 while retaining 77.6\% accuracy at 5.18$\times$ speedup. \diffsent{Because DSCD's supervision is decoupled from any specific inference policy, \modelname{} stacks directly on cache-based per-step accelerators: \modelname{}+d$^{2}$Cache and \modelname{}+dLLM-Cache combine NFE reduction (from CAD) with per-step compute reduction (from caching), reaching \textbf{8.96$\times$} (78.0\%) and \textbf{9.73$\times$} (77.6\%) wall-clock speedup and Pareto-dominating the standalone entries reproduced in the same harness in Table~\ref{tab:gsm8k}.} The composed gain is sub-multiplicative rather than the product of the two axes, since CAD has already compressed the schedule and leaves fewer steps over which the per-step cache can amortize. The same composition transfers to MBPP, where DSCD+Fast-dLLM Pareto-improves the standalone cache on both pass@1 (33.8\% $\to$ 38.7\%) and speedup (3.07$\times$ $\to$ 3.29$\times$, Appendix~\ref{sec:appendix-mbpp}).

Across the seven settings in Table~\ref{tab:benchmark}, \modelname{} reduces mean NFE from 292.6 to 117.2 (60\% less) and yields a 3.62$\times$ end-to-end speedup while slightly improving the unweighted average score from 45.6 to 46.3, with per-benchmark gains of +2.2pt pass@1 at 3.30$\times$ on HumanEval and +1.3pt accuracy at 5.33$\times$ on MATH500.

\paragraph{Uncertainty of accuracy and efficiency.}
\leavevmode\diffsent{Paired-bootstrap confidence intervals show that the headline GSM8K accuracy difference is $+0.2$ points with a 95\% confidence interval (CI) of $[-1.0,+1.3]$.} \diffsent{Across three training seeds, \modelname{} obtains $77.5\pm0.4$ accuracy and $76.8\pm1.9$ NFE, while HumanEval pass@5 varies by $\pm0.9$ points across five decoding seeds.} \diffsent{These results support accuracy preservation rather than a significant accuracy gain; paired NFE and latency differences are significant at $p<0.001$.}

\begin{table}[t]
\centering
\small
\setlength{\tabcolsep}{4pt}
\renewcommand{\arraystretch}{1.05}
\begin{tabular}{lccc}
\toprule
\textbf{Configuration} & \textbf{Acc (\%)} & \textbf{NFE} & \textbf{Speedup} \\
\midrule
\multicolumn{4}{l}{\emph{DSCD composes with inference accelerators}} \\
DSCD only (fixed)            & \textbf{78.3} & 256.0 & 1.00$\times$ \\
DSCD + CAD (Ours)            & 77.6 & 76.3  & 5.18$\times$ \\
DSCD + Fast-dLLM             & 77.2 & 70.5  & \textbf{5.60$\times$} \\
\midrule
\multicolumn{4}{l}{\emph{SFT-based training is constrained}} \\
SFT only (fixed)             & 77.6 & 256.0 & 1.00$\times$ \\
SFT + CAD                    & 77.6 & 95.4  & 4.15$\times$ \\
SFT + Fast-dLLM              & 76.9 & 83.7  & 4.73$\times$ \\
\bottomrule
\end{tabular}
\caption{\textbf{Decoding flexibility and training effectiveness on GSM8K} ($L_{\text{gen}}=256$, $b=32$). \emph{Group 1} shows that the same DSCD student operates with both CAD (training-aware decoding) and Fast-dLLM (KV caching) without retraining. \emph{Group 2} replaces DSCD with plain supervised fine-tuning (SFT) while holding the decoder fixed. \diffsent{Under both compared decoders, DSCD achieves a higher speedup than SFT at similar accuracy (CAD: 5.18$\times$ vs.\ 4.15$\times$; Fast-dLLM: 5.60$\times$ vs.\ 4.73$\times$).}}
\label{tab:component}
\end{table}

\paragraph{DSCD improves confidence-guided decoding efficiency.}
Table~\ref{tab:component} isolates DSCD's contribution by holding the decoding policy fixed (CAD or Fast-dLLM, mutually exclusive commit policies) and swapping the training objective. Under either decoder, the DSCD student preserves teacher-level accuracy at a higher speedup than the SFT student (CAD: 5.18$\times$ vs.\ 4.15$\times$; Fast-dLLM: 5.60$\times$ vs.\ 4.73$\times$). \diffsent{These matched-decoder comparisons are consistent with DSCD better supporting confidence-guided commitment.} \diffsent{Direct calibration measurements support this comparison: On held-out corrupted reference states, expected calibration error decreases from 0.064 for SFT to 0.042 for DSCD, and on-policy committed-token precision at $\gamma_{\mathrm{conf}}=0.95$ increases from 98.0\% to 98.9\%.} \diffsent{Restricting the KL term to positions masked for both teacher and student preserves similar accuracy (77.3\% vs.\ 77.6\%) but increases NFE from 76.3 to 84.2 and reduces committed-token precision to 98.4\%.} \diffsent{Under the diagnostic global-denoising setting ($b{=}L_{\text{gen}}$) without permanent EOS suppression, SFT+CAD reaches 38.2\% at 3.13$\times$, whereas DSCD+CAD reaches 54.7\% at 4.91$\times$ (Appendix~\ref{sec:appendix-ablation}).}

\paragraph{Component controls support the core design.}
\leavevmode\diffsent{Sampling teacher and student masks independently causes frequent loss divergence with gradient NaNs within the first epoch, whereas teacher-subset masking avoids this failure mode (Appendix~\ref{app:ablation_stability}).} \diffsent{Among convergent variants, removing the reconstruction loss retains 76.8\% accuracy but raises NFE from 76.3 to 98.7; replacing confidence-ranked commitment with uniform random selection is more damaging, reducing accuracy from 77.6\% to 50.6\% and raising NFE to 202.2 (Appendix~\ref{sec:appendix-ablation}).} \diffsent{Together, these controls isolate the two parts of our story: Nested-mask supervision and ground-truth anchoring make the student trainable and efficient across irregular states, while confidence ranking converts that signal into safe parallel commitments.}

\paragraph{EOS handling and fair pure-diffusion comparison.}
\leavevmode\diffsent{CAD suppresses positions whose top-ranked prediction is EOS until at least $\beta_{\text{EOS}}L_{\text{gen}}$ target positions have been committed; we use $\beta_{\text{EOS}}=0.3$.} \diffsent{If all eligible positions are blocked, the $k_{\min}$ fallback still commits tokens to guarantee progress.} \diffsent{This dynamic guardrail is released later in decoding, whereas the permanent EOS-suppression protocol sets $p(\texttt{<EOS>})=0$ throughout generation.} \diffsent{For a fair pure-diffusion comparison, both LLaDA and the DSCD student use permanent EOS suppression: LLaDA reaches 73.2\% and the same DSCD checkpoint reaches 77.4\% with CAD without retraining from the $b=32$ setting.} \diffsent{Without permanent EOS suppression, the diagnostic LLaDA and DSCD+CAD results are 13.8\% and 54.7\%, respectively (Appendix~\ref{sec:appendix-ablation}).} \diffsent{All settings retain the same $L_{\text{gen}}=256$ canvas; hence, shorter outputs do not explain the reported compute reduction.}

\paragraph{Consistent efficiency gains.}
Fig.~\ref{fig:bench} visualizes the accuracy-compute frontier across all four benchmarks: at the plotted budgets, our CAD ($\gamma_{\text{conf}}=0.95$) shifts the frontier upward and leftward versus the LLaDA baseline across diverse domains without per-benchmark tuning. On GSM8K, sweeping $\gamma_{\text{conf}} \in [0.85, 0.99]$ further produces a Pareto-improved frontier over the plotted baseline points (Appendix~\ref{sec:appendix-tradeoff}). The ablation in Appendix~\ref{sec:appendix-ablation} additionally shows that, under diagnostic global denoising, DSCD+CAD exceeds SFT+CAD in both accuracy (54.7\% vs.\ 38.2\%) and speedup (4.91$\times$ vs.\ 3.13$\times$).

\subsection{Mechanism of Efficiency}
\label{sec:efficiency_mechanism}

\begin{table}[t]
\caption{\diffsent{\textbf{GSM8K latency decomposition} (batch size 1; identical prompts, hardware, full canvas, and harness).} \diffsent{Similar output lengths rule out EOS length as the main source of the gap between NFE reduction and end-to-end speedup.}}
\label{tab:latency-breakdown}
\centering
\footnotesize
\setlength{\tabcolsep}{3pt}
\begin{tabular}{lcc}
\toprule
\textbf{Metric} & \textbf{LLaDA fixed} & \textbf{DSCD+CAD} \\
\midrule
Avg. NFE & 256.0 & 76.3 \\
Avg. output length to EOS & 141 & 138 \\
Processed target canvas & 256 & 256 \\
Model forward / step & 128 ms & 83 ms \\
Total latency / step & 133 ms & 87 ms \\
CAD controller / step & -- & 0.6 ms ($<1\%$) \\
End-to-end / request & 34.1 s & 6.6 s \\
\bottomrule
\end{tabular}
\end{table}

We analyze per-sample compute distributions in Appendix~\ref{sec:appendix-nfe-hist} (Fig.~\ref{fig:nfe_hist}). Across the four plotted tasks, CAD terminates at lower mean NFE than the fixed baseline budgets (means of 76.3, 148.3, 97.7, and 113.2 versus fixed budgets of 256, 512, 256, and 256, respectively).

\paragraph{Translating step reduction to wall-clock speedup.}
Our wall-clock speedup can exceed the NFE-ratio reduction: on GSM8K, \modelname{} uses 76.3 rather than 256.0 NFE (3.36$\times$ fewer) while reaching 5.18$\times$ measured wall-clock speedup (Table~\ref{tab:gsm8k}). Thus, NFE reduction alone does not account for the measured latency gain; the observed throughput also reflects per-step execution cost. \diffsent{Table~\ref{tab:latency-breakdown} shows that the fixed loop averages 133 ms per step and the CAD loop averages 87 ms per step, while both process the same full canvas and produce similar output lengths.} \diffsent{The remaining factor is therefore a decode-loop implementation effect rather than a training-side speedup; CAD controller overhead is below 1\% of per-step latency.}

\paragraph{One student spans mechanically different decoders.}
\leavevmode\diffsent{All composed entries reuse the same DSCD checkpoint: CAD and Fast-dLLM change the token-commit policy, whereas d$^{2}$Cache and dLLM-Cache reduce per-step compute; none requires retraining.} \diffsent{The checkpoint also switches from block diffusion at $b=32$ (77.6\%, 5.18$\times$) to pure diffusion at $b=L_{\text{gen}}$ under permanent EOS suppression (77.4\%, 4.91$\times$).} \diffsent{Thus, the evidence for decoding-pattern-agnostic training is not tied to one favorable decoder: accuracy remains stable across multiple inference paths that induce different intermediate states and compute profiles.}

\paragraph{Backbone generality.}
Applied to DiffuCoder-Instruct~\citep{gong2025diffucoder} with no architecture-specific modifications, the same DSCD+CAD pipeline improves MBPP pass@1 from 73.1\% to 73.7\% with a 3.25$\times$ wall-clock speedup (Appendix~\ref{sec:appendix-diffucoder}), providing evidence of transfer to a second absorbing-state DLM backbone.

\begin{table}[t]
  \caption{\diffsent{\textbf{One-factor CAD sweeps on GSM8K.}} \diffsent{Lower confidence thresholds reach more aggressive budgets but expose an accuracy--NFE boundary; $k_{\max}$ is comparatively stable in accuracy.}}
  \label{tab:cad-sensitivity}
  \centering
  \small
  \begin{tabular}{lccc}
    \toprule
    \textbf{Knob} & \textbf{Value} & \textbf{Acc. (\%)} & \textbf{NFE} \\
    \midrule
    $\gamma_{\mathrm{conf}}$ & 0.95 & 77.6 & 76.3 \\
     & 0.80 & 61.3 & 31.2 \\
     & 0.60 & 42.1 & 15.8 \\
    \midrule
    $k_{\max}$ & 8 & 77.7 & 117 \\
     & 16 & 77.6 & 88 \\
     & 32 & 77.6 & 76 \\
    \bottomrule
  \end{tabular}
\end{table}

\begin{table}[t]
  \caption{\diffsent{\textbf{Calibration and KL-support analysis on GSM8K.}} \diffsent{Precision is measured on committed tokens at $\gamma_{\mathrm{conf}}=0.95$.}}
  \label{tab:calibration-support}
  \centering
  \scriptsize
  \setlength{\tabcolsep}{2.5pt}
  \begin{tabular}{lcccc}
    \toprule
    \textbf{Target} & \textbf{ECE $\downarrow$} & \textbf{Prec. $\uparrow$} & \textbf{Acc.} & \textbf{NFE $\downarrow$} \\
    \midrule
    SFT + CAD & 0.064 & 98.0 & 77.6 & 95.4 \\
    DSCD ($\mathcal{M}_T$) & -- & 98.4 & 77.3 & 84.2 \\
    DSCD ($\mathcal{M}_S$) & \textbf{0.042} & \textbf{98.9} & 77.6 & \textbf{76.3} \\
    \bottomrule
  \end{tabular}
\end{table}

\paragraph{Decoding sensitivity.}
\leavevmode\diffsent{Accuracy remains within one point across $b\in\{16,32,64\}$ (77.2\%, 77.6\%, and 76.8\%).} \diffsent{In contrast, increasing $k_{\min}$ from 1 to 2 reduces NFE by about 12\% but costs about three accuracy points; larger values move toward the fixed-tps failure regime.} \diffsent{We, therefore, use $k_{\min}=1$.} \diffsent{Lowering $\gamma_{\mathrm{conf}}$ enables 31.2 and 15.8 NFE at 61.3\% and 42.1\% accuracy, respectively, showing an empirical, task-dependent boundary rather than unlimited NFE compression.}

\section{Conclusion}
\diffsent{We presented \modelname{}, a training-side DLM accelerator that commits to no decoding pattern: DSCD trains a student over random teacher-subset mask pairs without fixing a decoding path, and CAD uses this flexibility for dynamic confidence-adaptive decoding.} The result reaches teacher-level accuracy on GSM8K (77.6\% at 5.18$\times$ speedup) and transfers to DiffuCoder-Instruct without modification. \diffsent{A paired bootstrap confirms that the $+0.2$-point GSM8K accuracy difference is within statistical uncertainty (95\% CI $[-1.0,+1.3]$), whereas the NFE and latency reductions are significant at $p<0.001$.} \diffsent{Direct calibration measurements further show that DSCD lowers expected calibration error (ECE) from 0.064 to 0.042 and raises committed-token precision from 98.0\% to 98.9\%, supporting more reliable confidence-guided commitment.} \diffsent{Because DSCD's supervision is decoupled from any specific inference policy, stacking it with per-step cache accelerators is orthogonal: \modelname{}+dLLM-Cache and \modelname{}+d$^{2}$Cache reach the table's highest speedup (\textbf{9.73$\times$} at 77.6\%) and highest accuracy (\textbf{78.0\%} at 8.96$\times$), respectively, Pareto-dominating the standalone accelerators reproduced in the same harness.}

\paragraph{Outlook.}
\leavevmode\diffsent{Decoupling training from any specific decoding pattern is a productive design discipline: Stacking DSCD with future inference-side accelerators and pairing decoding-pattern-agnostic supervision with on-policy refinement are natural next steps, with \modelname{} providing a stable training-side layer beneath the evolving inference-side ecosystem.}

\clearpage
\section*{Limitations and Discussion}
\label{sec:limitation}

We close by acknowledging where \modelname{} trades standalone speed for composability, the structural assumptions inherited from its block-diffusion backbone, and the scope of our evaluation.

\paragraph{Position relative to trajectory-locked distillation.}
\leavevmode\diffsent{Decoding-pattern-agnostic supervision trades some standalone speed for composability.} Among the training-side baselines reported in Table~\ref{tab:gsm8k}, dParallel~\citep{chen2025dparallel} and d$^{3}$LLM~\citep{qian2026d3llm} have higher reported/extrapolated standalone speedups than \modelname{} (8.53$\times$ and 6.96$\times$ vs.\ 5.18$\times$), but lower accuracy (76.7\% and 74.2\% vs.\ 77.6\%). \modelname{} reaches higher speedups only once composed with reproduced cache accelerators. Whether these composability gains generalize beyond the three evaluated inference-side accelerators (Fast-dLLM, d$^{2}$Cache, and dLLM-Cache) remains open.

\paragraph{Static canvas constraints.}
Like the LLaDA backbone and most block-diffusion models, our approach relies on a pre-defined maximum sequence length $L_{\text{gen}}$. Although CAD handles variable-length outputs via EOS blocking, computation is allocated for the fixed generation canvas, which can waste work on short outputs and limits the maximum generated length.

\paragraph{Evaluation scope.}
Our experiments cover LLaDA-8B-Instruct and one DiffuCoder-Instruct transfer setting, evaluated on code and mathematical reasoning tasks. We do not test larger models or open-ended generation. CAD uses maximum token probability as a proxy for correctness; its behavior when many continuations are valid, such as in dialogue or creative writing, remains untested.

\clearpage

\bibliography{custom}

@article{nie2025large,
  title={Large language diffusion models},
  author={Nie, Shen and Zhu, Fengqi and You, Zebin and Zhang, Xiaolu and Ou, Jingyang and Hu, Jun and Zhou, Jun and Lin, Yankai and Wen, Ji-Rong and Li, Chongxuan},
  journal={arXiv preprint arXiv:2502.09992},
  year={2025}
}

@article{achiam2023gpt,
  title={{GPT}-4 technical report},
  author={Achiam, Josh and Adler, Steven and Agarwal, Sandhini and Ahmad, Lama and Akkaya, Ilge and Aleman, Florencia Leoni and Almeida, Diogo and Altenschmidt, Janko and Altman, Sam and Anadkat, Shyamal and others},
  journal={arXiv preprint arXiv:2303.08774},
  year={2023}
}

@article{grattafiori2024llama,
  title={The {Llama} 3 herd of models},
  author={Grattafiori, Aaron and Dubey, Abhimanyu and Jauhri, Abhinav and Pandey, Abhinav and Kadian, Abhishek and Al-Dahle, Ahmad and Letman, Aiesha and Mathur, Akhil and Schelten, Alan and Vaughan, Alex and others},
  journal={arXiv preprint arXiv:2407.21783},
  year={2024}
}

@article{cobbe2021training,
  title={Training verifiers to solve math word problems},
  author={Cobbe, Karl and Kosaraju, Vineet and Bavarian, Mohammad and Chen, Mark and Jun, Heewoo and Kaiser, Lukasz and Plappert, Matthias and Tworek, Jerry and Hilton, Jacob and Nakano, Reiichiro and others},
  journal={arXiv preprint arXiv:2110.14168},
  year={2021}
}

@article{hendrycks2021measuring,
  title={Measuring mathematical problem solving with the math dataset},
  author={Hendrycks, Dan and Burns, Collin and Kadavath, Saurav and Arora, Akul and Basart, Steven and Tang, Eric and Song, Dawn and Steinhardt, Jacob},
  journal={arXiv preprint arXiv:2103.03874},
  year={2021}
}

@article{austin2021program,
  title={Program synthesis with large language models},
  author={Austin, Jacob and Odena, Augustus and Nye, Maxwell and Bosma, Maarten and Michalewski, Henryk and Dohan, David and Jiang, Ellen and Cai, Carrie and Terry, Michael and Le, Quoc and others},
  journal={arXiv preprint arXiv:2108.07732},
  year={2021}
}

@article{chen2021evaluating,
  title={Evaluating large language models trained on code},
  author={Chen, Mark and Tworek, Jerry and Jun, Heewoo and Yuan, Qiming and Pinto, Henrique Ponde De Oliveira and Kaplan, Jared and Edwards, Harri and Burda, Yuri and Joseph, Nicholas and Brockman, Greg and others},
  journal={arXiv preprint arXiv:2107.03374},
  year={2021}
}

@article{austin2021structured,
  title={Structured denoising diffusion models in discrete state-spaces},
  author={Austin, Jacob and Johnson, Daniel D. and Ho, Jonathan and Tarlow, Daniel and Van Den Berg, Rianne},
  journal={Advances in Neural Information Processing Systems},
  volume={34},
  pages={17981--17993},
  year={2021}
}

@article{shi2024simplified,
  title={Simplified and generalized masked diffusion for discrete data},
  author={Shi, Jiaxin and Han, Kehang and Wang, Zhe and Doucet, Arnaud and Titsias, Michalis},
  journal={Advances in Neural Information Processing Systems},
  volume={37},
  pages={103131--103167},
  year={2024}
}

@inproceedings{wei2025accelerating,
  title={Accelerating diffusion large language models with slowfast sampling: The three golden principles},
  author={Wei, Qingyan and Zhang, Yaojie and Liu, Zhiyuan and Zeng, Puyu and Wang, Yuxuan and Qi, Biqing and Liu, Dongrui and Zhang, Linfeng},
  booktitle={International Conference on Learning Representations},
  volume={2026},
  pages={4340--4354},
  year={2026}
}

@article{sahoo2024simple,
  title={Simple and effective masked diffusion language models},
  author={Sahoo, Subham and Arriola, Marianne and Schiff, Yair and Gokaslan, Aaron and Marroquin, Edgar and Chiu, Justin and Rush, Alexander and Kuleshov, Volodymyr},
  journal={Advances in Neural Information Processing Systems},
  volume={37},
  pages={130136--130184},
  year={2024}
}

@article{israel2025accelerating,
  title={Accelerating Diffusion {LLMs} via Adaptive Parallel Decoding},
  author={Israel, Daniel and Broeck, Guy Van den and Grover, Aditya},
  journal={arXiv preprint arXiv:2506.00413},
  year={2025}
}

@inproceedings{christopher2025speculative,
  title={Speculative diffusion decoding: Accelerating language generation through diffusion},
  author={Christopher, Jacob K. and Bartoldson, Brian R. and Ben-Nun, Tal and Cardei, Michael and Kailkhura, Bhavya and Fioretto, Ferdinando},
  booktitle={Proceedings of the 2025 Conference of the Nations of the Americas Chapter of the Association for Computational Linguistics: Human Language Technologies (Volume 1: Long Papers)},
  pages={12042--12059},
  year={2025}
}

@article{liu2025dllm,
  title={{dLLM-Cache}: Accelerating diffusion large language models with adaptive caching},
  author={Liu, Zhiyuan and Yang, Yicun and Zhang, Yaojie and Chen, Junjie and Zou, Chang and Wei, Qingyan and Wang, Shaobo and Zhu, Yichen and Zhang, Linfeng},
  journal={arXiv preprint arXiv:2506.06295},
  year={2025}
}

@article{wu2025fast,
  title={{Fast-dLLM}: Training-free acceleration of diffusion {LLM} by enabling {KV} cache and parallel decoding},
  author={Wu, Chengyue and Zhang, Hao and Xue, Shuchen and Liu, Zhijian and Diao, Shizhe and Zhu, Ligeng and Luo, Ping and Han, Song and Xie, Enze},
  journal={arXiv preprint arXiv:2505.22618},
  year={2025}
}

@article{gong2025diffucoder,
  title={{DiffuCoder}: Understanding and Improving Masked Diffusion Models for Code Generation},
  author={Gong, Shansan and Zhang, Ruixiang and Zheng, Huangjie and Gu, Jiatao and Jaitly, Navdeep and Kong, Lingpeng and Zhang, Yizhe},
  journal={arXiv preprint arXiv:2506.20639},
  year={2025}
}

@article{ye2025dream,
  title={Dream {7B}: Diffusion large language models},
  author={Ye, Jiacheng and Xie, Zhihui and Zheng, Lin and Gao, Jiahui and Wu, Zirui and Jiang, Xin and Li, Zhenguo and Kong, Lingpeng},
  journal={arXiv preprint arXiv:2508.15487},
  year={2025}
}

@article{ahmad2025opencodeinstruct,
  title={{OpenCodeInstruct}: A Large-scale Instruction Tuning Dataset for Code {LLMs}},
  author={Ahmad, Wasi Uddin and Ficek, Aleksander and Samadi, Mehrzad and Huang, Jocelyn and Noroozi, Vahid and Majumdar, Somshubra and Ginsburg, Boris},
  journal={arXiv preprint arXiv:2504.04030},
  year={2025}
}

@article{hinton2015distilling,
  title={Distilling the knowledge in a neural network},
  author={Hinton, Geoffrey and Vinyals, Oriol and Dean, Jeff},
  journal={arXiv preprint arXiv:1503.02531},
  year={2015}
}

@article{li2025beyond,
  title={Beyond fixed: Training-free variable-length denoising for diffusion large language models},
  author={Li, Jinsong and Dong, Xiaoyi and Zang, Yuhang and Cao, Yuhang and Wang, Jiaqi and Lin, Dahua},
  journal={arXiv preprint arXiv:2508.00819},
  year={2025}
}

@article{chen2025dparallel,
  title={{dParallel}: Learnable Parallel Decoding for {dLLMs}},
  author={Chen, Zigeng and Fang, Gongfan and Ma, Xinyin and Yu, Ruonan and Wang, Xinchao},
  journal={arXiv preprint arXiv:2509.26488},
  year={2025}
}

@article{kim2026consistency,
  title={{CDLM}: Consistency Diffusion Language Models For Faster Sampling},
  author={Kim, Minseo and Xu, Chenfeng and Hooper, Coleman and Singh, Harman and Athiwaratkun, Ben and Zhang, Ce and Keutzer, Kurt and Gholami, Amir},
  journal={arXiv preprint arXiv:2511.19269},
  year={2025}
}

@article{qian2026d3llm,
  title={{d3LLM}: Ultra-Fast Diffusion {LLM} using Pseudo-Trajectory Distillation},
  author={Qian, Yu-Yang and Su, Junda and Hu, Lanxiang and Zhang, Peiyuan and Deng, Zhijie and Zhao, Peng and Zhang, Hao},
  journal={arXiv preprint arXiv:2601.07568},
  year={2026}
}

@article{zhang2026t3d,
  title={Few-Step Diffusion Language Models via Trajectory Self-Distillation},
  author={Zhang, Tunyu and Zhang, Xinxi and Han, Ligong and Shi, Haizhou and He, Xiaoxiao and Li, Zhuowei and Wang, Hao and Xu, Kai and Srivastava, Akash and Mao, Chengzhi and others},
  journal={arXiv preprint arXiv:2602.12262},
  year={2026}
}

@article{jiang2025d,
  title={{d$^{2}$Cache}: Accelerating Diffusion-Based {LLMs} via Dual Adaptive Caching},
  author={Jiang, Yuchu and Cai, Yue and Luo, Xiangzhong and Fu, Jiale and Wang, Jiarui and Liu, Chonghan and Yang, Xu},
  journal={arXiv preprint arXiv:2509.23094},
  year={2025}
}

@inproceedings{deschenaux2024beyond,
  title={Beyond autoregression: Fast {LLMs} via self-distillation through time},
  author={Deschenaux, Justin and Gulcehre, Caglar},
  booktitle={International Conference on Learning Representations},
  volume={2025},
  pages={5007--5045},
  year={2025}
}

@article{wang2025d2f,
  title={Diffusion {LLMs} can do faster-than-{AR} inference via discrete diffusion forcing},
  author={Wang, Xu and Xu, Chenkai and Jin, Yijie and Jin, Jiachun and Zhang, Hao and Deng, Zhijie},
  journal={arXiv preprint arXiv:2508.09192},
  year={2025}
}

\clearpage

\appendix
\counterwithin{figure}{section}
\counterwithin{table}{section}

\renewcommand{\thefigure}{\thesection.\arabic{figure}}
\renewcommand{\thetable}{\thesection.\arabic{table}}

\section{Preliminaries}
\label{sec:appendix-preliminaries}

We review the absorbing-state discrete diffusion formulation adopted from LLaDA \citep{nie2025large} to establish the notation used in our theoretical analysis.

\paragraph{Data representation.}
Let $\mathcal{D}=\{(x,y)\}$ be a corpus of conditional generation tasks, where $x=(x_1,\dots,x_{L_x})$ is the prompt (e.g., a coding problem or mathematics question) and $y=(y_1,\dots,y_{L_y})$ is the target sequence (e.g., the solution program or derivation). We concatenate them into a single sequence $z=(x;y)=(z_1,\dots,z_L)$, with $L=L_x+L_y$. The base vocabulary $\mathcal{V}$ is augmented with a special absorbing mask token $\mathtt{m}\notin\mathcal{V}$, and an end-of-sequence token \texttt{<EOS>}, giving the extended vocabulary: $\mathcal{V}^+=\mathcal{V}\cup\{\mathtt{m},\texttt{<EOS>}\}$. For chain-of-thought supervision, we place the entire reasoning trace together with the final answer in the target region ($i>L_x$); hence, the model explicitly learns to refine both intermediate reasoning and final answers. For code completion benchmarks, the natural-language problem description and function signature belong to the prompt region, while the full solution body is treated as target tokens.

\paragraph{Forward absorbing process.}
We adopt the absorbing discrete diffusion framework for text \citep{austin2021structured,shi2024simplified,sahoo2024simple,nie2025large}, which models corruption as stochastic erasure rather than additive noise. Let $t\in[0,1]$ be a continuous noise level with a monotone schedule $\alpha(t)\in[0,1]$ satisfying $\alpha(0)=0$ and $\alpha(1)\approx 1$. The forward process $q_t(\tilde{z}\mid z)$ gradually destroys information in the target region while preserving the prompt:
\begin{equation*}
q_t(\tilde{z}\mid z)=\prod_{i=1}^{L} q_t(\tilde{z}_i\mid z_i)
\label{eq:forward-process2}
\end{equation*}
\begin{equation}
\resizebox{.88\columnwidth}{!}{$
q_t(\tilde{z}_i\mid z_i)=
\begin{cases}
\delta_{\tilde{z}_i=z_i}, & i\le L_x,\\[2pt]
(1-\alpha(t))\delta_{\tilde{z}_i=z_i} + \alpha(t)\delta_{\tilde{z}_i=\mathtt{m}},
& i> L_x,
\end{cases}
$}
\label{eq:forward-process}
\end{equation}
where $\delta$ is the Kronecker delta. Once a token is replaced by $\mathtt{m}$, it remains masked for all larger $t$; hence, the forward chain is an absorbing Markov process. We denote the masked set by $\mathcal{M}(\tilde{z})=\{i:\tilde{z}_i=\mathtt{m}\}$. In practice, we do not simulate this Markov chain step-by-step. Given $z$ and $t$, we instead sample a Bernoulli mask $b\in\{0,1\}^{L_y}$ with $\Pr(b_j=1)=\alpha(t)$ for target positions and set $\tilde{z}=\mathrm{Mask}(z,b)$, which matches $q_t(\tilde{z}\mid z)$ in \eqref{eq:forward-process}. In all experiments, we reuse the continuous-time masking schedule from LLaDA and adopt the same choices of $\alpha(t)$, the sampling distribution $\rho(t)$ over $t$, and the weighting function $w(t)$.

\paragraph{Diffusion teacher and ELBO-style objective.}
A DLM parameterized by $\phi$ defines a reverse model that predicts clean tokens at the masked positions of a corrupted sequence:
\begin{equation}
p_\phi(z\mid\tilde{z},t)
=\prod_{i\in\mathcal{M}(\tilde{z})} p_\phi(z_i\mid\tilde{z},t),
\end{equation}
where each $p_\phi(z_i\mid\tilde{z},t)$ is a categorical distribution over $\mathcal{V}^+$. We follow LLaDA~\citep{nie2025large} and parameterize $p_\phi$ with a Transformer that maps $(\tilde{z},t)$ to logits over the vocabulary at all positions; unmasked positions are simply copied from the input. For absorbing discrete diffusion, the log-likelihood admits an evidence lower bound whose dominant term reduces to a time-weighted masked cross-entropy \citep{austin2021structured,sahoo2024simple,shi2024simplified}. We adopt the simplified continuous-time objective used in masked DLMs:
\begin{equation}
\begin{aligned}
\mathcal{L}_{\text{teacher}}(\phi)
&=
\mathbb{E}_{(x,y)\sim\mathcal{D}}
\mathbb{E}_{t\sim\rho(t)}
\mathbb{E}_{\tilde{z}\sim q_t(\cdot\mid z)}
\bigg[
\frac{w(t)}{|\mathcal{M}(\tilde{z})|}\\
&\qquad\cdot
\sum_{i\in\mathcal{M}(\tilde{z})}
-\log p_\phi(z_i\mid\tilde{z},t)
\bigg],
\end{aligned}
\label{eq:teacher-loss}
\end{equation}
where $\rho(t)$ and $w(t)$ are inherited from LLaDA. Optimizing \eqref{eq:teacher-loss} yields a high-quality diffusion teacher, but sampling requires iterating over a long schedule $t_1>\dots>t_T$, leading to high decoding latency. In practice, $p_\phi$ is instantiated as the publicly released LLaDA checkpoint of the corresponding size and kept frozen throughout student training; we do not modify its architecture, training objective, or noise schedule. All improvements reported in this article, therefore, come from the distilled student and our decoding policy, rather than from retraining the teacher.

\paragraph{Notation.}
We summarize the main symbols in the theoretical analysis and their meanings.
\begin{itemize}[leftmargin=*,itemsep=2pt]
    \item $p_\theta, p_\phi$ = Student model and frozen Teacher model
    \item $\tilde{z}^S, \tilde{z}^T$(Sect.~\ref{sec:method:dscd}) $\leftrightarrow$ $S, T$ (Appendix~\ref{sec:appendix-dscd-theory}) = masked input views
    \item $\mathcal{M}_S, \mathcal{M}_T$ = masked position sets for student and teacher
    \item $\lambda(g), \tau$ = Curriculum schedule and distillation temperature
    \item $\mathcal{M}^{(s)}$ (Sect.~\ref{sec:method:cad}) = masked set at decoding step $s$
    \item $\mathcal{E}^{(s)}$ = eligible set = $\mathcal{M}^{(s)} \cap \mathcal{B}_j$ (active block intersection)
    \item $\mathcal{S}^{(s)}$ = commit set (positions to unmask at step $s$)
    \item $c_i^{(s)}$ = confidence score = $\max_v p_\theta(v \mid \tilde{z}^{(s)})_i$
    \item $\gamma_{\text{conf}}$ = Confidence threshold for CAD acceptance
    \item $k_{\min}, k_{\max}$ = Trust region bounds for commit size $|\mathcal{S}^{(s)}|$
\end{itemize}

\onecolumn
\section{A Conditional Projection View of DSCD}
\label{sec:appendix-dscd-theory}

\diffsent{This section gives an idealized conditional-projection view of DSCD.}
\diffsent{For a fixed student state and an unrestricted student family, minimizing $\mathcal{L}_{\text{cons}}$ yields the conditional average of teacher targets over compatible nested teacher views.}
\diffsent{This result motivates the pairwise surrogate but does not guarantee trajectory invariance across all timesteps or mask distributions.}

\subsection{Setup and Information Ordering}

Let $(\Omega, \mathcal{F}, P)$ be the probability space that supports all random variables in our construction. Let $Z \in \mathcal{Z}$ denote the clean (ground-truth) token sequence. The absorbing masking process of LLaDA defines, for each masking pattern $\mathcal{M}$, a partially masked sequence $X^{(\mathcal{M})}$ obtained by replacing a subset of token positions in $Z$ with the mask token $\mathtt{m}$.

To simplify notation in this theoretical analysis, we denote the random variables corresponding to the student view $\tilde{z}^S$ and teacher view $\tilde{z}^T$ (defined in Sec.~\ref{sec:method:dscd}) simply as $S$ and $T$. We define:
\[ S = X^{(\mathcal{M}_S)}, \qquad T = X^{(\mathcal{M}_T)}. \]

We write
\[
\mathcal{F}_S = \sigma(S), \qquad \mathcal{F}_T = \sigma(T)
\]
for the $\sigma$-algebras generated by the corresponding observed (unmasked) tokens.

In our nested masking design, the teacher always sees a superset of the tokens visible to the student. Concretely, if $\mathcal{M}_S$ and $\mathcal{M}_T$ denote the random sets of \emph{masked} positions for the student and teacher, we enforce
\[
\mathcal{M}_T \subseteq \mathcal{M}_S,
\]
so that every token unmasked for the student is also unmasked for the teacher. Equivalently, the teacher's observation $T$ is a measurable function of $(S,Z)$, and the associated $\sigma$-algebras are ordered as $\mathcal{F}_S \subseteq \mathcal{F}_T$.

\begin{definition}[Information ordering]
\label{def:info_order}
We say that the teacher is informationally richer than the student, and write $S \preceq T$, if
\[
\mathcal{F}_S \subseteq \mathcal{F}_T,
\]
i.e., every event that is measurable with respect to the student state $S$ is also measurable with respect to the teacher state $T$. Under our nested masking scheme $\mathcal{M}_T \subseteq \mathcal{M}_S$, we have $S \preceq T$ by construction.
\end{definition}

\subsection{Consistency as Projection onto Student Information}
We now formalize how DSCD can be viewed as projecting the teacher's conditional distribution onto the information available to the student. Let $Z \in \mathcal{Z}$ denote the clean (ground-truth) sequence, and let $S$ and $T$ be the student and teacher states as in Definition~\ref{def:info_order}, with $\mathcal{F}_S = \sigma(S)$ and $\mathcal{F}_T = \sigma(T)$ satisfying $S \preceq T$. 

Given a student state $S = s$, the DSCD objective compares the student distribution $p_\theta(\cdot \mid s)$ against the teacher distribution $p_\phi(\cdot \mid T)$ averaged over the teacher states $T$ that are compatible with $S = s$. Formally, the DSCD consistency loss can be written as
\begin{equation}
\label{eq:dscd_loss_simple}
\mathcal{L}_{\text{DSCD}}(\theta)=\mathbb{E}_{S}\mathbb{E}_{T \mid S}\big[D_{\text{KL}}(p_\phi(\cdot \mid T)\|p_\theta(\cdot \mid S))\big],
\end{equation}
where the outer expectation is taken over the student states induced by the diffusion process and masking policy, and the inner expectation is taken over the corresponding teacher states. For later use, we define the teacher-induced \emph{projected} distribution
\begin{equation}
\label{eq:Q_def}
Q(z \mid S = s):=\mathbb{E}_{T \mid S = s}[p_\phi(z \mid T)],
\end{equation}
i.e., $Q(\cdot \mid s)$ is obtained by averaging the teacher's conditional distributions over all teacher states $T$ consistent with the student state $S = s$. Equivalently, $Q(\cdot \mid S)$ is the conditional expectation of $p_\phi(\cdot \mid T)$ with respect to the $\sigma$-algebra $\mathcal{F}_S$.
\begin{proposition}[DSCD as projection onto $\mathcal{F}_S$]
\label{prop:dscd_projection}
Assume the student model class $\{p_\theta(\cdot \mid S)\}$ is rich enough to represent any conditional distribution over $\mathcal{Z}$ for each $S = s$. Then the DSCD loss \eqref{eq:dscd_loss_simple} is minimized by
\begin{equation}
\label{eq:optimal_student}
p_\theta^*(z \mid S = s)=Q(z \mid S = s)=\mathbb{E}_{T \mid S = s}[p_\phi(z \mid T)],\quad \forall s\in\mathcal{S}.
\end{equation}
In particular, $p_\theta^*(\cdot \mid S)$ coincides with the conditional expectation of the teacher's conditional distribution $p_\phi(\cdot \mid T)$ onto the coarser $\sigma$-algebra $\mathcal{F}_S$.
\end{proposition}
\begin{proof}
By expanding the KL divergence in \eqref{eq:dscd_loss_simple}, we obtain
\begin{align}
\mathcal{L}_{\text{DSCD}}(\theta) &= \mathbb{E}_{S,T}\!\left[D_{\text{KL}}(p_\phi(\cdot \mid T)\|p_\theta(\cdot \mid S))\right] \\
&= \mathbb{E}_{S,T}\!\left[\sum_{z\in\mathcal{Z}} p_\phi(z \mid T)\log\frac{p_\phi(z \mid T)}{p_\theta(z \mid S)}\right] \\
&= \underbrace{\mathbb{E}_{S,T}\!\left[\sum_{z\in\mathcal{Z}} p_\phi(z \mid T)\log p_\phi(z \mid T)\right]}_{\text{constant w.r.t. }\theta} - \mathbb{E}_{S,T}\!\left[\sum_{z\in\mathcal{Z}} p_\phi(z \mid T)\log p_\theta(z \mid S)\right].
\end{align}
The first term does not depend on $\theta$ and can be dropped. Thus, minimizing \eqref{eq:dscd_loss_simple} is equivalent to maximizing
\begin{equation}
\label{eq:max_cross_entropy}
\mathbb{E}_{S,T}\!\left[\sum_{z\in\mathcal{Z}} p_\phi(z \mid T)\log p_\theta(z \mid S)\right] = \mathbb{E}_{S}\!\left[\sum_{z\in\mathcal{Z}} \mathbb{E}_{T \mid S}[p_\phi(z \mid T)] \log p_\theta(z \mid S)\right].
\end{equation}
For each fixed student state $S=s$, the inner objective becomes
\begin{equation}
\sum_{z\in\mathcal{Z}} Q(z \mid s)\log p_\theta(z \mid s),
\end{equation}
which is maximized when $p_\theta(\cdot \mid s)=Q(\cdot \mid s)$. Therefore, $p_\theta^*(\cdot\mid S)=Q(\cdot\mid S)$.
\end{proof}
\diffsent{This proposition shows that, for each fixed student state under the stated model-class assumption, the optimum need not reproduce a particular teacher trajectory.} \diffsent{Instead, it matches the teacher distribution averaged over compatible richer views, i.e., the projection onto the coarser information $\sigma$-algebra $\mathcal{F}_S$.}

\section{An Idealized Posterior-Risk View of Teacher-Subset Masking}
\label{app:variance}

\diffsent{We analyze the \emph{teacher-subset masking} scheme ($\mathcal{M}_T \subseteq \mathcal{M}_S$) as an information-nesting design.}
Intuitively, if $\mathcal{M}_T$ were independent of $\mathcal{M}_S$, the teacher might lack access to tokens visible to the student, causing it to marginalize over ``blind'' positions and produce high-variance targets.
\diffsent{By enforcing $\mathcal{M}_T \subseteq \mathcal{M}_S$, we ensure the teacher always conditions on a superset of the student's information, typically a strict superset before integer rounding.}
\diffsent{Under an idealized Bayes predictor, the corresponding conditional-expectation result states that posterior residual risk cannot increase as the information set is refined.} \diffsent{Independent masking has no such nesting property, because it may hide tokens that the student already observes.}

\subsection{Problem Setup}
Let $z$ be the ground truth sequence. The student observes a partial view $\tilde{z}^S$ with mask $\mathcal{M}_S$. The student's goal is to minimize the divergence from a target distribution provided by the teacher.
Let the teacher's view be $\tilde{z}^T$ with mask $\mathcal{M}_T$. The teacher's prediction is a random variable $Y = p_\phi(\cdot \mid \tilde{z}^T)$, which serves as the regression target for the student.

We compare two masking schemes:
\begin{itemize}
    \item \textbf{Teacher-subset masking:} $\mathcal{M}_T \subseteq \mathcal{M}_S$. The teacher observes everything the student observes, plus additional tokens $\Delta = \mathcal{M}_S \setminus \mathcal{M}_T$.
    \item \diffsent{\textbf{Independent masking:} $\mathcal{M}_T$ is sampled independently. Crucially, there can be a non-empty set of ``blind'' positions $\mathcal{B} = \mathcal{M}_T \cap (\{L_x+1,\ldots,L\} \setminus \mathcal{M}_S)$ that are \emph{visible to the student} but \emph{masked for the teacher}.}
\end{itemize}

\subsection{Posterior Risk Reduction via Rao-Blackwellization}

We now make the risk-reduction claim precise. Let $U$ be a random target quantity that the teacher aims to approximate (e.g., a one-hot encoding of the ground-truth token or a sufficient statistic thereof). For any information $\sigma$-algebra $\mathcal{G}$, the Bayes-optimal predictor based on $\mathcal{G}$ is the conditional expectation
\[
  Y^{(\mathcal{G})} = \mathbb{E}[U \mid \mathcal{G}].
\]

The student observes a state $\tilde{z}^S$ and the corresponding $\sigma$-algebra $\mathcal{F}_S = \sigma(\tilde{z}^S)$. In the teacher-subset scheme, the teacher observes all information available to the student plus additional tokens. Denote a nested teacher information set by $\mathcal{F}_T$, with $\mathcal{F}_S \subseteq \mathcal{F}_T$.

While the frozen teacher model $p_\phi$ is not the true Bayes-optimal posterior of the data, the conditional-expectation calculation clarifies why the nested design is statistically preferable: It keeps the teacher target compatible with the student's observed tokens and only adds information beyond the student view.

The following lemma is a conditional version of the Rao-Blackwell theorem.

\begin{lemma}[Conditional Rao-Blackwell risk reduction]
\label{lem:conditional_rb}
Let $U \in L^2(\Omega,\mathcal{F},P)$ and let $\mathcal{G} \subseteq \mathcal{H} \subseteq \mathcal{F}$ be $\sigma$-algebras. Define
\[
  Y_{\mathcal{G}} = \mathbb{E}[U \mid \mathcal{G}], \qquad Y_{\mathcal{H}} = \mathbb{E}[U \mid \mathcal{H}].
\]
Then for any $\sigma$-algebra $\mathcal{K} \subseteq \mathcal{G}$,
\[
  \mathbb{E}\!\left[\|U-Y_{\mathcal{H}}\|_2^2 \mid \mathcal{K}\right]
  \le
  \mathbb{E}\!\left[\|U-Y_{\mathcal{G}}\|_2^2 \mid \mathcal{K}\right]
  \quad\text{almost~surely~(a.s.)}
\]
with equality iff $Y_{\mathcal{G}} = Y_{\mathcal{H}}$ a.s.
\end{lemma}

\begin{proof}
Fix $\mathcal{K} \subseteq \mathcal{G}$. Since $\mathcal{G} \subseteq \mathcal{H}$, the tower property gives $\mathbb{E}[Y_{\mathcal{H}}\mid\mathcal{G}] = Y_{\mathcal{G}}$. The orthogonality of conditional expectations in $L^2$ yields the Pythagorean identity
\[
  \mathbb{E}\!\left[\|U-Y_{\mathcal{G}}\|_2^2 \mid \mathcal{K}\right]
  =
  \mathbb{E}\!\left[\|U-Y_{\mathcal{H}}\|_2^2 \mid \mathcal{K}\right]
  +
  \mathbb{E}\!\left[\|Y_{\mathcal{H}}-Y_{\mathcal{G}}\|_2^2 \mid \mathcal{K}\right].
\]
The second term is nonnegative and vanishes iff $Y_{\mathcal{H}}=Y_{\mathcal{G}}$ a.s., proving the claim.
\end{proof}

We now instantiate Lemma~\ref{lem:conditional_rb} with
\[
  \mathcal{K} = \mathcal{F}_S,\quad
  \mathcal{G} = \mathcal{F}_S,\quad
  \mathcal{H} = \mathcal{F}_T.
\]
By construction, $\mathcal{F}_S \subseteq \mathcal{F}_T$, and therefore
\[
  \mathbb{E}\!\left[\|U-\mathbb{E}[U\mid\mathcal{F}_T]\|_2^2 \mid \mathcal{F}_S\right]
  \le
  \mathbb{E}\!\left[\|U-\mathbb{E}[U\mid\mathcal{F}_S]\|_2^2 \mid \mathcal{F}_S\right].
\]
The inequality is strict whenever the extra tokens revealed to the teacher carry information about $U$ beyond the student view.

\diffsent{Under this idealization, teacher-subset masking is analogous to Rao--Blackwellization relative to the student's information set: It preserves all tokens visible to the student and adds only extra information from the same clean sequence.} Independent masking lacks this monotone information structure; it can create teacher targets that are blind to tokens the student observes, which is precisely the instability mode diagnosed in Appendix~\ref{app:ablation_stability}.

\section{Theoretical Analysis of Confidence-Adaptive  Decoding (CAD)}
\label{sec:appendix-cad-theory}

This appendix provides (i) a decision-theoretic motivation for CAD as a greedy solver of a per-step risk-efficiency trade-off with trust-region and liveness constraints, and (ii) basic guarantees on termination, step bounds, and computational overhead. 

This derivation shows that confidence thresholding is optimal for an additive risk-efficiency objective, and that under cardinality constraints, the optimal solution is a confidence-sorted prefix. CAD implements a greedy approximation with explicit trust-region and liveness constraints.

\subsection{Notation and Eligible Set}
Let the target length be $L_{\text{gen}}$. At decoding step $s$, the current partially-masked state is $\tilde{z}^{(s)}\in\mathcal{V}^{L_{\text{gen}}}$, and the masked index set is
$\mathcal{M}^{(s)}=\{i\in[L_{\text{gen}}]: \tilde{z}^{(s)}_i=\mathtt{m}\}$.
For block diffusion with block size $b$, define blocks
\[
\mathcal{B}_j=\{(j-1)b+1,\ldots,\min(jb,L_{\text{gen}})\},\quad j=1,\ldots,J,\; J=\lceil L_{\text{gen}}/b\rceil,
\]
and maintain an active block index $j$ (decoded left-to-right).
CAD operates on the \emph{eligible set}
\begin{equation}
\mathcal{E}^{(s)} \;=\; \mathcal{M}^{(s)} \cap \mathcal{B}_j,
\label{eq:eligible-set-app}
\end{equation}
i.e., masked positions within the active block. Pure diffusion is recovered as a special case by setting $b=L_{\text{gen}}$ (a single block), in which case $\mathcal{E}^{(s)}=\mathcal{M}^{(s)}$.

\subsection{Decision-theoretic Derivation (Per-step Objective)}
At step $s$, the model produces a confidence score for each eligible position $i\in\mathcal{E}^{(s)}$,
$\;c_i^{(s)}\in[0,1]$.
We interpret $c_i^{(s)}$ as a proxy for the probability that the current argmax token at position $i$ is correct:
\begin{equation}
    c_i^{(s)} \approx \mathbb{P}(\text{correct}_i\mid \tilde{z}^{(s)}).
    \label{eq:calibrated-confidence-app}
\end{equation}
Let $\mathcal{V}^{(s)}\subseteq \mathcal{E}^{(s)}$ denote the \emph{valid} candidates at step $s$ (e.g., after excluding padding or other structurally invalid positions; by default $\mathcal{V}^{(s)}=\mathcal{E}^{(s)}$).
The decision variable is a commit set $\mathcal{S}\subseteq\mathcal{V}^{(s)}$.

We define two opposing per-step objectives:
(i) \textbf{risk}: the expected number of incorrect tokens committed at this step,
$\mathcal{R}(\mathcal{S})=\sum_{i\in\mathcal{S}}(1-c_i^{(s)})$,
and (ii) \textbf{efficiency}: the number of committed tokens,
$\mathcal{E}(\mathcal{S})=|\mathcal{S}|$.
We combine them into a net cost
\begin{equation}
    \mathcal{L}(\mathcal{S})
    \;=\; \mathcal{R}(\mathcal{S}) - \lambda_{\text{trade}}\,\mathcal{E}(\mathcal{S})
    \;=\; \sum_{i\in\mathcal{S}}\big(1-c_i^{(s)}-\lambda_{\text{trade}}\big),
    \label{eq:obj-app}
\end{equation}
where $\lambda_{\text{trade}}>0$ controls the risk-speed trade-off and is \emph{unrelated} to the training loss-mixing schedule $\lambda(g)$.
Because \eqref{eq:obj-app} is additive across positions, the unconstrained minimizer commits all positions whose per-token contribution is negative:
\begin{equation}
    i\in\mathcal{S}^*
    \iff 1-c_i^{(s)}-\lambda_{\text{trade}}<0
    \iff c_i^{(s)} > 1-\lambda_{\text{trade}}.
\end{equation}
Defining the confidence threshold $\gamma_{\text{conf}}=1-\lambda_{\text{trade}}$ yields the basic \emph{confidence thresholding} rule.

\subsection{Trust-region and Liveness Constraints, and the CAD Rule}
In practice, we impose two cardinality constraints:
\emph{trust region} ($|\mathcal{S}|\le k_{\max}$) to prevent committing too many tokens based on local confidence estimates, and
\emph{liveness} ($|\mathcal{S}|\ge k_{\min}$) to guarantee progress whenever possible.
This yields the constrained problem
\begin{equation}
    \min_{\mathcal{S}\subseteq \mathcal{V}^{(s)}} \mathcal{L}(\mathcal{S})
    \quad \text{s.t.}\quad
    k_{\min} \le |\mathcal{S}| \le k_{\max}.
    \label{eq:constrained-obj-app}
\end{equation}
Sorting candidates by decreasing confidence produces an ordering $i_{(1)},\ldots,i_{(|\mathcal{V}^{(s)}|)}$ with
$c_{i_{(1)}}^{(s)}\ge\cdots\ge c_{i_{(|\mathcal{V}^{(s)}|)}}^{(s)}$.
Under \eqref{eq:constrained-obj-app}, the optimum is always a prefix of this ordering.
CAD adopts a simple greedy approximation: Count how many positions within the trust region exceed $\gamma_{\text{conf}}$,
\begin{equation}
    n_{\text{conf}}
    \;=\;
    \left|\left\{j \le \min(k_{\max},|\mathcal{V}^{(s)}|): c_{i_{(j)}}^{(s)}\ge \gamma_{\text{conf}}\right\}\right|,
    \label{eq:n-conf-app}
\end{equation}
then clamp the committed token count and select the top-$k^{(s)}$ positions:
\begin{equation}
    k^{(s)}
    \;=\;
    \min\Big(\max(k_{\min},\,n_{\text{conf}}),\; \min(k_{\max},\,|\mathcal{V}^{(s)}|)\Big),
    \qquad
    \mathcal{S}^{(s)}=\{i_{(1)},\ldots,i_{(k^{(s)})}\}.
    \label{eq:k-update-app}
\end{equation}
Equations~\eqref{eq:n-conf-app}-\eqref{eq:k-update-app} match the CAD update implemented in the main text.

\subsection{Monotonic Progress and Step Bound}

\paragraph{Termination guarantee.}
Within the active block $\mathcal{B}_j$, CAD operates in an \emph{absorbing} manner: once a position is committed (i.e., \texttt{<mask>} is replaced by a concrete token), it is never remasked.
Let $\mathcal{E}^{(s)} = \mathcal{M}^{(s)} \cap \mathcal{B}_j$ denote the set of unresolved (masked) positions in the active block at step $s$.
The clamping mechanism in Eq.~\eqref{eq:k-update-app} enforces $k^{(s)} \ge 1$ whenever $\mathcal{E}^{(s)} \neq \emptyset$ (typically $k^{(s)} \ge k_{\min}$). 
Consequently, CAD guarantees strict monotonicity in the reduction of the mask set:
\begin{equation}
    |\mathcal{M}^{(s+1)} \cap \mathcal{B}_j| < |\mathcal{M}^{(s)} \cap \mathcal{B}_j|.
\end{equation}
This implies that any block of size $|\mathcal{B}_j|$ must be fully resolved after at most $N_{\text{step}} = \lceil |\mathcal{B}_j|/k_{\min} \rceil$ function evaluations. 
This yields a deterministic per-block latency upper bound, ruling out the non-terminating refinement loops often observed in non-absorbing iterative decoding.

\paragraph{EOS blocking and safety.}
\label{sec:appendix-eos}
\leavevmode\diffsent{To prevent premature termination, we apply an \textbf{EOS blocking} strategy with a decoded-length ratio $\beta_{\text{EOS}}$.} \diffsent{While $n_{\text{decoded}}/L_{\text{gen}}<\beta_{\text{EOS}}$, eligible positions whose top-ranked prediction is EOS are excluded from confidence-adaptive selection whenever at least one non-EOS position remains.} \diffsent{If all eligible positions are blocked, the $k_{\min}$ fallback still commits positions to guarantee progress; after the decoded-length threshold is reached, EOS predictions are again eligible.} \diffsent{Because every update still commits at least one position when masks remain, this guardrail preserves the monotonicity property derived above.}

The following statements formalize the monotonicity and yield the step bound.

\begin{lemma}[Monotonicity within the active block]
\label{lem:monotone-app}
Assume $\mathcal{E}^{(s)}\neq \emptyset$ and $k_{\min}\ge 1$. After one CAD update,
\[
|\mathcal{M}^{(s+1)}\cap \mathcal{B}_j| \;\le\; |\mathcal{M}^{(s)}\cap \mathcal{B}_j| - 1.
\]
Moreover, if $|\mathcal{E}^{(s)}|\ge k_{\min}$, then
\[
|\mathcal{M}^{(s+1)}\cap \mathcal{B}_j| \;\le\; |\mathcal{M}^{(s)}\cap \mathcal{B}_j| - k_{\min}.
\]
\end{lemma}

\begin{proof}
By construction, $\mathcal{S}^{(s)}\subseteq \mathcal{E}^{(s)}=\mathcal{M}^{(s)}\cap\mathcal{B}_j$ and each $i\in\mathcal{S}^{(s)}$ is committed from \texttt{<mask>} to a concrete token, hence removed from $\mathcal{M}^{(s)}$.
If $\mathcal{E}^{(s)}\neq\emptyset$, the clamp in \eqref{eq:k-update-app} ensures $k^{(s)}\ge 1$, yielding the first inequality.
When $|\mathcal{E}^{(s)}|\ge k_{\min}$, we have $k^{(s)}\ge k_{\min}$, yielding the second inequality.
\end{proof}

\begin{corollary}[Per-block and total NFE bound]
\label{cor:stepbound-app}
For block size $b$, each block $\mathcal{B}_j$ is resolved within at most $\left\lceil |\mathcal{B}_j|/k_{\min}\right\rceil$ NFE.
Therefore, the total NFE satisfies
\[
S \;\le\; \sum_{j=1}^{J}\left\lceil\frac{|\mathcal{B}_j|}{k_{\min}}\right\rceil
\;\le\;
\left\lceil\frac{L_{\text{gen}}}{k_{\min}}\right\rceil + J
\;\le\;
\left\lceil\frac{L_{\text{gen}}}{k_{\min}}\right\rceil + \left\lceil\frac{L_{\text{gen}}}{b}\right\rceil,
\]
and is additionally capped by the implementation limit $S_{\max}$.
\end{corollary}

\subsection{Special Cases and Relation to Fixed-step Decoding}
\begin{proposition}[Pure diffusion as a special case]
Setting $b=L_{\text{gen}}$ (i.e., $J=1$) reduces block diffusion with CAD to pure diffusion with CAD, because $\mathcal{E}^{(s)}=\mathcal{M}^{(s)}$ for all $s$.
\end{proposition}

\begin{proposition}[Reduction to fixed-step block diffusion]
If $\gamma_{\text{conf}}=-\infty$ and $k_{\min}=k_{\max}=k$ (a constant), then CAD commits exactly $k$ positions per NFE within the active block (until fewer than $k$ masks remain), matching a fixed-step block diffusion schedule up to the last step of each block.
\end{proposition}

\subsection{Complexity and Overhead}
\label{sec:appendix-complexity}
A natural concern for CAD is the computational cost of the control policy itself (e.g., selecting high-confidence positions and updating the active set), especially at large batch sizes.
We denote by $B$ the batch size, $L$ the total sequence length (prompt + target), $L_{\text{gen}}$ the maximum generation budget, $D$ the hidden dimension, and $N_{\text{layer}}$ the number of Transformer layers.
Ignoring constant factors, a single forward pass of the student model has time complexity
\begin{equation}
    \text{Cost}_{\text{model}}
    = \mathcal{O}\!\big(B\,N_{\text{layer}}(L^2D + LD^2)\big),
\end{equation}
dominated by self-attention and feed-forward blocks over the full context.

In contrast, the CAD controller operates only on token-wise confidence scores within the eligible set. At decoding step $s$, it performs thresholding and/or a top-$k$ operation over at most $|\mathcal{E}^{(s)}|$ elements, costing
\begin{equation}
    \text{Cost}_{\text{CAD}}^{(s)}
    = \mathcal{O}\!\big(B\,|\mathcal{E}^{(s)}|\log|\mathcal{E}^{(s)}|\big).
\end{equation}
Under block diffusion with block size $b$, we have $|\mathcal{E}^{(s)}|\le b$ for all $s$, hence
\begin{equation}
    \text{Cost}_{\text{CAD}}
    = \sum_{s=1}^{S}\text{Cost}_{\text{CAD}}^{(s)}
    = \mathcal{O}\!\big(BS\,b\log b\big)
    \;\le\; \mathcal{O}\!\big(BS\,L_{\text{gen}}\log L_{\text{gen}}\big),
\end{equation}
where the last inequality follows from $b \le L_{\text{gen}}$ (pure diffusion corresponds to $b = L_{\text{gen}}$).
Comparing the two, the relative overhead satisfies:
\begin{equation}
    \frac{\text{Cost}_{\text{CAD}}}{\text{Cost}_{\text{model}}}
    \;\lesssim\;
    \frac{B S \, b \log b}{B \, N_{\text{layer}} S (L^2 D + L D^2)}
    \;\approx\; \mathcal{O}\!\left(\frac{b \log b}{N_{\text{layer}} L^2 D}\right).
\end{equation}
Since $L$ represents the full sequence length, we observe two asymptotic regimes:
(1) For block diffusion ($b$ fixed, $b \ll L$), the overhead decays quadratically with sequence length ($\propto L^{-2}$), rendering it negligible for long-context generation.
(2) Even for pure diffusion where $b \approx L_{\text{gen}} < L$, the ratio scales as $\mathcal{O}(L^{-1})$, which still vanishes asymptotically as model depth and width increase.

\section{Method Details}

\subsection{Distillation Algorithm Pseudocode}
\label{sec:appendix-distillation}

Algorithm~\ref{alg:dscd} describes the full training procedure for DSCD. The algorithm alternates between (i) constructing paired student-teacher views via teacher-subset masking, (ii) computing the curriculum-weighted loss, and (iii) updating the student parameters.

\begin{algorithm}[h]
\caption{Discrete-Space Consistency Distillation (DSCD)}
\label{alg:dscd}
\begin{algorithmic}[1]
\Require Teacher model $p_\phi$ (frozen), student model $p_\theta$, dataset $\mathcal{D}$
\Require Mask ratio ranges: $[r_S^{\min}, r_S^{\max}]$, $[r_T^{\min}, r_T^{\max}]$
\Require Distillation temperature $\tau$, curriculum parameters $\lambda_0, \lambda_1, \alpha$
\Require Maximum training steps $G_{\max}$
\For{$g = 1$ \textbf{to} $G_{\max}$}
    \State Sample mini-batch $\{(x^{(b)}, y^{(b)})\}_{b=1}^{B} \sim \mathcal{D}$
    \For{$b = 1$ \textbf{to} $B$}
        \State \diffsent{$z \gets (x^{(b)}; y^{(b)})$} \Comment{Concatenate prompt and target}
        \State \diffsent{$L_x \gets |x^{(b)}|$, $L_y \gets |y^{(b)}|$, $L \gets L_x + L_y$}
        \Statex \textcolor{gray}{\textit{// --- Step 1: Sample mask ratios (see Eq.~\ref{eq:mask-ratio-sampling}) ---}}
        \State $r_S \sim \mathrm{Uniform}(r_S^{\min}, r_S^{\max})$
        \State $u \sim \mathrm{Uniform}(0.3, 0.7)$
        \State $r_T \gets \mathrm{clip}(r_S \cdot u, \; r_T^{\min}, \; r_T^{\max})$
        
        \Statex \textcolor{gray}{\textit{// --- Step 2: Construct student mask ---}}
        \State \diffsent{$\mathcal{U}_{\text{valid}} \gets \{L_x+1, \ldots, L\}$}
        \State \diffsent{$n_S \gets \min\!\big(|\mathcal{U}_{\text{valid}}|,\max(1,\lfloor |\mathcal{U}_{\text{valid}}| \cdot r_S \rfloor)\big)$}
        \State $\mathcal{M}_S \gets \text{RandomSample}(\mathcal{U}_{\text{valid}}, \; n_S)$ \Comment{Uniformly sample positions}
        \State $\tilde{z}^S \gets \text{Mask}(z, \mathcal{M}_S)$ \Comment{Replace $z_i$ with $\mathtt{m}$ for $i \in \mathcal{M}_S$}
        
        \Statex \textcolor{gray}{\textit{// --- Step 3: Teacher-subset masking ---}}
        \State \diffsent{$n_T \gets \min(\lfloor |\mathcal{U}_{\text{valid}}| \cdot r_T \rfloor, \; n_S)$}
        \State \diffsent{$\mathcal{M}_T \gets \text{RandomSubset}(\mathcal{M}_S, \; n_T)$} \Comment{Nested subset: $\mathcal{M}_T \subseteq \mathcal{M}_S$}
        \State $\tilde{z}^T \gets \text{Mask}(z, \mathcal{M}_T)$
        
        \Statex \textcolor{gray}{\textit{// --- Step 4: Forward pass ---}}
        \State $\ell_\theta \gets p_\theta(\cdot \mid \tilde{z}^S)$ \Comment{Student logits}
        \State $\ell_\phi \gets p_\phi(\cdot \mid \tilde{z}^T)$ \Comment{Teacher logits (no grad)}
        \Statex \textcolor{gray}{\textit{// --- Step 5: Compute losses on $\mathcal{M}_S$ ---}}
        
        \State \diffsent{$\ell_{\text{recon}}^{(b)} \gets -\frac{1}{|\mathcal{M}_S|} \sum_{i \in \mathcal{M}_S} \log p_\theta(z_i \mid \tilde{z}^S)$}
        \State $\tilde{p}_\phi \gets \mathrm{softmax}(\ell_\phi / \tau)$ \Comment{Temperature-scaled teacher}
        \State \diffsent{$\ell_{\text{cons}}^{(b)} \gets \frac{\tau^2}{|\mathcal{M}_S|} \sum_{i \in \mathcal{M}_S} \mathrm{KL}(\tilde{p}_{\phi,i} \| p_{\theta,i})$}
    \EndFor
    \Statex \textcolor{gray}{\textit{// --- Step 6: Curriculum-weighted total loss ---}}
    \State \diffsent{$\mathcal{L}_{\text{recon}} \gets \frac{1}{B}\sum_{b=1}^{B}\ell_{\text{recon}}^{(b)},\quad \mathcal{L}_{\text{cons}} \gets \frac{1}{B}\sum_{b=1}^{B}\ell_{\text{cons}}^{(b)}$}
    \State $\lambda(g) \gets \text{CosineSchedule}(g/G_{\max}; \lambda_0, \lambda_1, \alpha)$ \Comment{Eq.~\eqref{eq:lambda-schedule}}
    \State $\mathcal{L}_{\text{total}} \gets \lambda(g) \cdot \mathcal{L}_{\text{cons}} + (1 - \lambda(g)) \cdot \mathcal{L}_{\text{recon}}$
    \Statex \textcolor{gray}{\textit{// --- Step 7: Parameter update ---}}
    \State \diffsent{$\theta \gets \theta - \eta \cdot \nabla_\theta \mathcal{L}_{\text{total}}$}
\EndFor
\State \Return Trained student $p_\theta$
\end{algorithmic}
\end{algorithm}
\vspace*{2mm}
\paragraph{Key design choices.}
\begin{enumerate}[leftmargin=*,itemsep=2pt]
    \item \diffsent{\textbf{Teacher-subset masking} (Step 3):} By enforcing $\mathcal{M}_T \subseteq \mathcal{M}_S$, the teacher always observes a superset of tokens visible to the student. This Rao-Blackwellizes the distillation target (Appendix~\ref{app:variance}).
    \item \diffsent{\textbf{Multiplicative ratio coupling}: The teacher ratio is derived from $r_S \cdot u$ with $u \sim \mathcal{U}(0.3, 0.7)$ and then clipped to its configured range. The resulting mask sets satisfy $|\mathcal{M}_T|\leq|\mathcal{M}_S|$ after integer rounding.}
    \item \diffsent{\textbf{Minimum mask count}: The student view masks at least one valid target position, preventing a zero-denominator reconstruction or consistency loss on short targets.}
\end{enumerate}

\paragraph{Mask ratio sampling.}
The mask ratios are sampled according to:
\begin{equation}
\begin{aligned}
r_S &\sim \mathcal{U}(r_S^{\min}, r_S^{\max}) = \mathcal{U}(0.4, 0.9), \\
r_T &= \mathrm{clip}\big(r_S \cdot u, \; r_T^{\min}, \; r_T^{\max}\big), \quad u \sim \mathcal{U}(0.3, 0.7),
\end{aligned}
\label{eq:mask-ratio-sampling}
\end{equation}
where $[r_T^{\min}, r_T^{\max}] = [0.1, 0.6]$. The multiplicative coupling ensures that the teacher always sees more context than the student (i.e., $|\mathcal{M}_T| \le |\mathcal{M}_S|$).

\paragraph{Curriculum schedule function.}
The mixing coefficient $\lambda(g)$ follows a cosine schedule with warmup:
\begin{equation}
\lambda(g) =
\begin{cases}
\lambda_0, & g \le \alpha, \\[3pt]
\lambda_0 + (\lambda_1 - \lambda_0) \cdot \dfrac{1 - \cos\!\big(\pi \frac{g - \alpha}{1 - \alpha}\big)}{2}, & g > \alpha,
\end{cases}
\label{eq:lambda-schedule}
\end{equation}
where $g \in [0,1]$ is the normalized training progress, $\alpha$ is the warmup fraction, and $\lambda_0 > \lambda_1$ ensures the transition from distillation-dominated to reconstruction-dominated training.

\subsection{Adaptive Decoding Algorithm Pseudocode}
\label{sec:appendix-decoding}

\begin{figure*}[t]
  \centering
  \includegraphics[width=\textwidth]{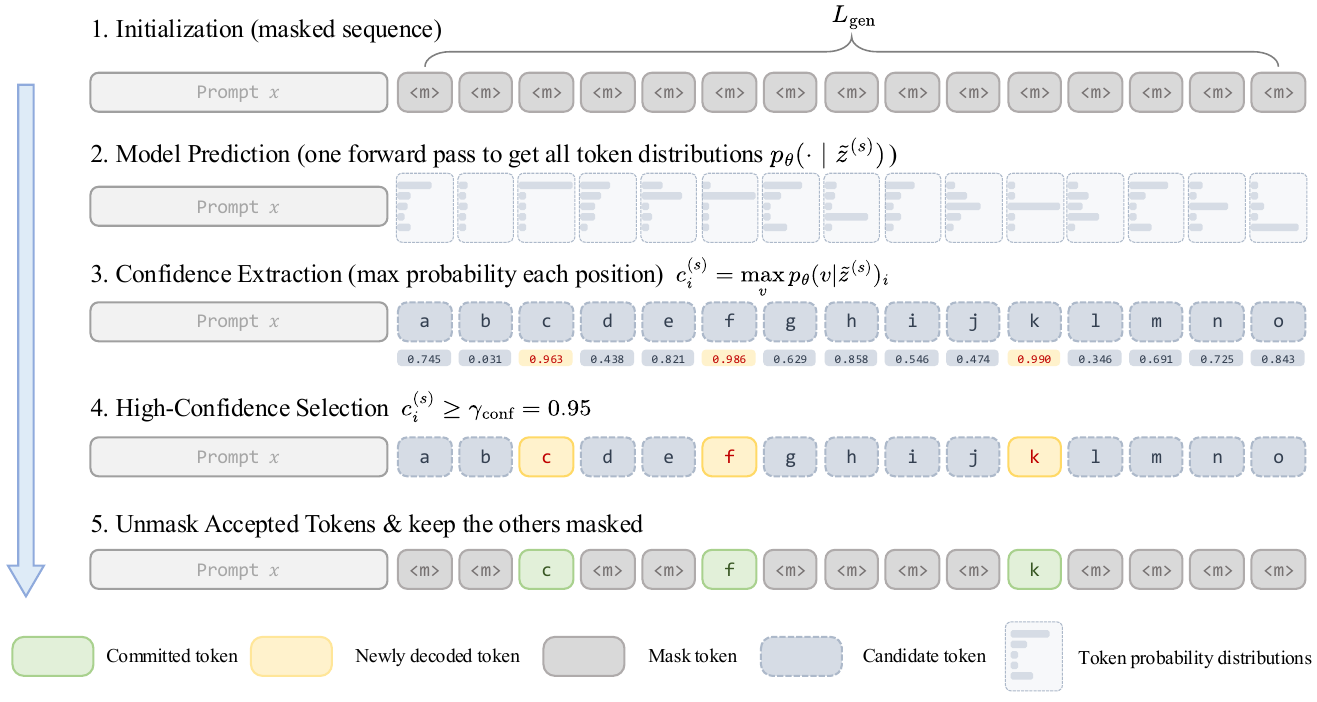}
  \caption{\textbf{Illustration of CAD.} In each step, the model predicts token distributions for all masked positions in parallel. We compute the confidence score $c_i^{(s)}$ based on the maximum probability and selectively unmask tokens that satisfy the threshold $c_i^{(s)} \ge \gamma_{\text{conf}}$. Low-confidence tokens remain masked for subsequent iterations. (Moved from main paper to appendix.)}
  \label{fig:CAD}
\end{figure*}

Algorithm~\ref{alg:cad} presents the CAD procedure. The algorithm operates within the block diffusion framework and adaptively selects how many tokens to commit at each step based on model confidence.

\begin{figure*}[t]
  \centering
  \includegraphics[width=\textwidth]{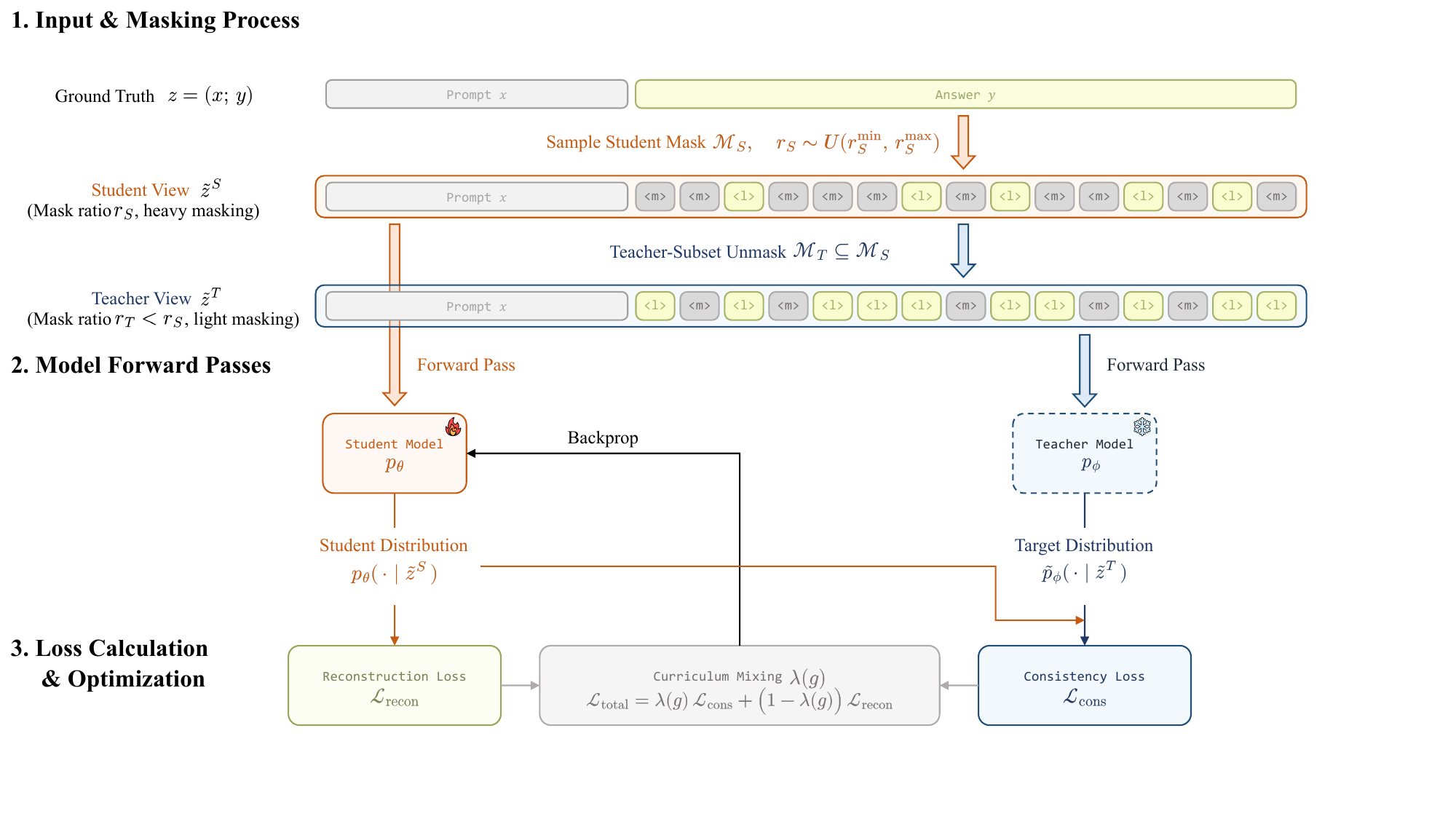}
    \caption{\textbf{Overview of different decoding paradigms.} 
      (a) Standard Autoregressive decoding. 
      (b) Fixed-schedule Block Diffusion (adopted by the \textbf{LLaDA backbone} \citep{nie2025large}), which decodes a preset number of masked tokens per step. 
      (c) Block diffusion with our \textbf{Confidence-Adaptive Decoding (CAD)}: Within each block, we adaptively select a high-confidence set of masked positions to unmask (implementation detailed in \textbf{Algorithm~\ref{alg:cad}}), enabling variable token counts across steps. 
      (d) The same CAD rule applied to \emph{pure} diffusion, iteratively denoising the entire sequence starting from all \texttt{<mask>} tokens.}
  \label{fig:decoding-overview}
\end{figure*}

\begin{algorithm}[h]
\caption{Confidence-Adaptive  Decoding (CAD)}
\label{alg:cad}
\begin{algorithmic}[1]
\Require Student model $p_\theta$, prompt $x$
\Require Maximum generation budget $L_{\text{gen}}$
\Require Block size $b$, confidence threshold $\gamma_{\text{conf}}$
\Require Acceptance bounds $k_{\min}, k_{\max}$, EOS blocking ratio $\beta_{\text{EOS}}$, max steps $S_{\max}$
\State $\tilde{z}^{(0)} \gets (x; \underbrace{\mathtt{m}, \ldots, \mathtt{m}}_{L_{\text{gen}}})$ \Comment{Initialize with all masks}
\State $\mathcal{M}^{(0)} \gets \{1, \ldots, L_{\text{gen}}\}$ \Comment{All target positions masked}
\State $J \gets \lceil L_{\text{gen}} / b \rceil$, $j \gets 1$ \Comment{Number of blocks, active block index}
\State $s \gets 0$, $n_{\text{decoded}} \gets 0$ \Comment{Step counter, decoded token count}
\While{$\mathcal{M}^{(s)} \neq \emptyset$ \textbf{and} $s < S_{\max}$}
    \Statex \textcolor{gray}{\textit{// --- Step 1: Compute eligible set within active block ---}}
    \State $\mathcal{B}_j \gets \{(j-1)b + 1, \ldots, \min(jb, L_{\text{gen}})\}$
    \State $\mathcal{E}^{(s)} \gets \mathcal{M}^{(s)} \cap \mathcal{B}_j$ \Comment{Eligible = masked $\cap$ active block}
    \If{$\mathcal{E}^{(s)} = \emptyset$}
        \State $j \gets j + 1$ \Comment{Advance to next block}
        \State \textbf{continue}
    \EndIf
    \Statex \textcolor{gray}{\textit{// --- Step 2: Model forward and confidence computation ---}}
    \State $\ell \gets p_\theta(\cdot \mid \tilde{z}^{(s)})$ \Comment{Logits for all positions}
    \State $p \gets \mathrm{softmax}(\ell)$ \Comment{Token probabilities}
    \State $\hat{z}_i \gets \arg\max_v \ell_{i,v}$ for $i \in \mathcal{E}^{(s)}$ \Comment{Predicted tokens}
    \State $c_i^{(s)} \gets p_{i, \hat{z}_i}$ for $i \in \mathcal{E}^{(s)}$ \Comment{Confidence = prob of argmax token}
    \Statex \textcolor{gray}{\textit{// --- Step 3: EOS blocking ---}}
    \If{$n_{\text{decoded}} / L_{\text{gen}} < \beta_{\text{EOS}}$}
        \For{$i \in \mathcal{E}^{(s)}$ where $\hat{z}_i \in \{\texttt{<EOS>}, \texttt{<|im\_end|>}, \ldots\}$}
            \State $c_i^{(s)} \gets -\infty$ \Comment{Suppress premature EOS}
        \EndFor
    \EndIf
    \Statex \textcolor{gray}{\textit{// --- Step 4: Confidence-adaptive selection ---}}
    \State $\mathcal{V}^{(s)} \gets \{i \in \mathcal{E}^{(s)} : c_i^{(s)} > -\infty\}$ \Comment{Valid (non-blocked) positions}
    \If{$\mathcal{V}^{(s)} = \emptyset$}
        \State $k^{(s)} \gets \min(k_{\min}, |\mathcal{E}^{(s)}|)$ \Comment{Force progress when all blocked}
        \State $\mathcal{S}^{(s)} \gets \mathrm{TopK}(\{c_i^{(s)}\}_{i \in \mathcal{E}^{(s)}}, k^{(s)})$
    \Else
        \State $n_{\text{conf}} \gets |\{i \in \mathcal{V}^{(s)} : c_i^{(s)} \ge \gamma_{\text{conf}}\}|$ \Comment{Count above threshold}
        \State \diffsent{$k^{(s)} \gets \min\!\big(\max(n_{\text{conf}}, k_{\min}),\min(k_{\max}, |\mathcal{V}^{(s)}|)\big)$}
        \State $\mathcal{S}^{(s)} \gets \mathrm{TopK}(\{c_i^{(s)}\}_{i \in \mathcal{V}^{(s)}}, k^{(s)})$ \Comment{Select top-$k$ positions}
    \EndIf
    \Statex \textcolor{gray}{\textit{// --- Step 5: Commit predictions ---}}
    \For{$i \in \mathcal{S}^{(s)}$}
        \State $\tilde{z}^{(s+1)}_i \gets \hat{z}_i$
    \EndFor
    \State $\mathcal{M}^{(s+1)} \gets \mathcal{M}^{(s)} \setminus \mathcal{S}^{(s)}$
    \State $n_{\text{decoded}} \gets n_{\text{decoded}} + |\mathcal{S}^{(s)}|$
    \State $s \gets s + 1$
\EndWhile
\State \Return $\tilde{z}^{(s)}$ \Comment{Fully decoded sequence}
\end{algorithmic}
\end{algorithm}

\paragraph{Key algorithmic properties.}
\begin{enumerate}[leftmargin=*,itemsep=2pt]
    \item \textbf{Block-wise decoding} (lines 7-12): Positions are partitioned into $J = \lceil L_{\text{gen}} / b \rceil$ contiguous blocks. CAD processes one block at a time, advancing to the next block only when the current one is fully resolved. Pure diffusion is recovered by setting $b = L_{\text{gen}}$.

    \item \textbf{Confidence computation} (line 17): The confidence score $c_i^{(s)} = p_{i, \hat{z}_i}$ is the probability mass assigned to the predicted token $\hat{z}_i = \arg\max_v p_{i,v}$. This equals the maximum probability $\max_v p_{i,v}$ by construction.

    \item \textbf{Confidence thresholding} (line 25): Positions with confidence $c_i^{(s)} \ge \gamma_{\text{conf}}$ are candidates for commitment. This implements the optimal decision rule derived from the risk-efficiency trade-off (Appendix~\ref{sec:appendix-cad-theory}).

    \item \textbf{Cardinality clamping} (line 26): The commit count $k^{(s)}$ satisfies:
    \begin{equation}
    \diffmath{k^{(s)} = \min\!\big(\max(n_{\text{conf}}, k_{\min}),\min(k_{\max}, |\mathcal{V}^{(s)}|)\big).}
    \end{equation}
    The lower bound $k_{\min}$ guarantees progress (liveness); the upper bound $k_{\max}$ prevents overcommitment from unreliable confidence estimates (trust region).

    \item \textbf{EOS blocking} (lines 18-21): Premature \texttt{<EOS>} predictions are suppressed until at least $\beta_{\text{EOS}} \cdot L_{\text{gen}}$ tokens have been decoded, preventing degenerate short outputs. When all positions predict EOS (lines 22-24), the algorithm forces progress by selecting $k_{\min}$ tokens regardless of blocking.
\end{enumerate}

\paragraph{Complexity analysis.}
Let $S$ be the total number of decoding steps. The CAD controller performs at most $\mathcal{O}(b \log b)$ operations per step (sorting within the eligible set), yielding total overhead $\mathcal{O}(S \cdot b \log b)$. This is negligible compared to the model forward cost $\mathcal{O}(S \cdot N_{\text{layer}}(L^2 D + L D^2))$, as detailed in Appendix~\ref{sec:appendix-cad-theory}.

\paragraph{Step bound.}
By Corollary~\ref{cor:stepbound-app}, the total NFE is bounded by:
\begin{equation}
S \le \left\lceil \frac{L_{\text{gen}}}{k_{\min}} \right\rceil + \left\lceil \frac{L_{\text{gen}}}{b} \right\rceil,
\end{equation}
plus the implementation cap $S_{\max}$. With default parameters $k_{\min} = 1$ and $b = L_{\text{gen}}$, this reduces to $S \le L_{\text{gen}} + 1$ in the worst case; in practice, adaptive acceptance yields $S \ll L_{\text{gen}}$.

\clearpage

\subsection{Implementation Details and Hyperparameters}
\label{sec:appendix-hyperparameters}

\paragraph{Training hyperparameters.}
Table~\ref{tab:hparam-train-full} summarizes all training hyperparameters with their symbols matching the notation in Sect.~\ref{sec:method:dscd}.

\begin{table}[h]
\centering
\small
\setlength{\tabcolsep}{6pt}
\begin{tabular}{lccc}
\toprule
\textbf{Hyperparameter} & \textbf{Symbol} & \textsc{Code} & \textsc{Math} \\
\midrule

\multicolumn{4}{l}{\textbf{\textit{Optimization:}}} \\
\quad Optimizer            & --                   & AdamW        & AdamW \\
\quad Learning rate        & $\eta$               & $5\times10^{-6}$ & $5\times10^{-6}$ \\
\quad LR Schedule          & --                   & cosine & cosine \\
\quad Weight decay         & --                   & 0.01         & 0.01 \\
\quad Betas                & $(\beta_1,\beta_2)$  & (0.9, 0.999) & (0.9, 0.999) \\
\quad Gradient clipping    & --                   & 1.0          & 1.0 \\
\quad \diffsent{Warmup steps (10\%)} & --          & \diffsent{300} & \diffsent{200} \\
\quad Effective batch size & $B$                  & 64           & 64 \\
\quad Training epochs      & --                   & 3            & 10 \\
\midrule

\multicolumn{4}{l}{\textbf{\textit{Distillation:}}} \\
\quad Temperature          & $\tau$               & 1.5          & 2.0 \\
\quad Initial $\lambda$    & $\lambda_0$          & 0.9          & 0.9 \\
\quad Final $\lambda$      & $\lambda_1$          & 0.5          & 0.5 \\
\quad Warmup ratio         & $\alpha$             & 0.1          & 0.1 \\
\quad Student mask ratio   & $r_{\mathrm S}\in[\cdot]$ & $[0.4,0.9]$ & $[0.4,0.9]$ \\
\quad Teacher mask ratio   & $r_{\mathrm T}\in[\cdot]$ & $[0.1,0.6]$ & $[0.1,0.6]$ \\
\quad Ratio coupling factor& $u$                  & $\mathcal{U}(0.3,0.7)$ & $\mathcal{U}(0.3,0.7)$ \\
\midrule

\multicolumn{4}{l}{\textbf{\textit{Data:}}} \\
\quad Dataset              & --                   & OpenCodeInstruct & GSM8K \\
\quad Training samples     & --                   & 190K & 7{,}473 \\
\bottomrule
\end{tabular}
\caption{\textbf{Training hyperparameters and notation.}}
\label{tab:hparam-train-full}
\end{table}

\paragraph{Inference hyperparameters.}
Table~\ref{tab:hparam-infer-full} summarizes all inference hyperparameters with their symbols matching the notation in Sect.~\ref{sec:method:cad}.

\begin{table}[h]
\centering
\small
\setlength{\tabcolsep}{6pt}
\begin{tabular}{lcc}
\toprule
\textbf{Hyperparameter} & \textbf{Symbol} & \textbf{Value} \\
\midrule

\multicolumn{3}{l}{\textbf{\textit{CAD Decoding:}}} \\
\quad Confidence threshold   & $\gamma_{\text{conf}}$ & 0.95 \\
\quad Min tokens per step    & $k_{\min}$             & 1 \\
\quad Max tokens per step    & $k_{\max}$             & 32 \\
\quad EOS blocking ratio     & $\beta_{\text{EOS}}$   & 0.3 \\
\quad Max generation budget  & $L_{\text{gen}}$             & 256 \\
\quad Max NFE               & $S_{\max}$             & 512 \\
\midrule

\multicolumn{3}{l}{\textbf{\textit{Block Diffusion:}}} \\
\quad Block size             & $b$                    & 32/256 \\
\midrule

\multicolumn{3}{l}{\textbf{\textit{Sampling:}}} \\
\quad Temperature (pass@1)   & $\tau_{\text{samp}}$   & 0 \\
\quad Temperature (pass@5)   & $\tau_{\text{samp}}$   & 1.0 \\
\midrule

\multicolumn{3}{l}{\textbf{\textit{Fixed-Step Baseline:}}} \\
\quad Diffusion steps        & $T$                    & 256 \\
\bottomrule
\end{tabular}
\caption{\textbf{Inference hyperparameters for CAD decoding.}}
\label{tab:hparam-infer-full}
\end{table}

\twocolumn
\section{Ablation Study}
\label{sec:appendix-ablation}

To disentangle the contributions of the proposed training and decoding components, we conduct a comprehensive ablation study. Although our main distillation experiments focus on code generation, we select \textbf{GSM8K} as the primary testbed for this analysis. GSM8K problems induce longer chain-of-thought sequences, making decoding artifacts easier to diagnose, and reasoning accuracy is highly sensitive to partial errors.
To ensure a strictly fair comparison regarding model capacity, we introduce an SFT baseline trained on the same GSM8K dataset using full fine-tuning, identical to our DSCD training configuration. This allows us to isolate the impact of the distillation objective from the benefits of standard supervised learning.

\subsection{Effectiveness and Efficiency Analysis}
Table~\ref{tab:ablation} presents the quantitative results. We draw two key conclusions regarding the superiority of DSCD over standard SFT:

\textbf{1. DSCD is essential for Global Denoising (Validity).} 
In the local regime ($b{=}32$), both SFT and DSCD perform comparably to the teacher, as the denoising task is relatively simple. However, the distinction becomes evident in the global regime ($b{=}256$) under the diagnostic setting without LLaDA's permanent EOS-suppression heuristic. The unmodified teacher collapses to 13.8\% accuracy because of premature \texttt{<EOS>} termination, and the SFT baseline improves accuracy only marginally to 21.9\%. In contrast, DSCD significantly boosts accuracy to 54.8\% under the same diagnostic setting. This implies that standard supervised learning captures the data distribution but fails to learn the \textit{many-to-many} dependency structure required for long-range parallel decoding.

\textbf{2. DSCD unlocks Aggressive Acceleration (Efficiency).}
Comparing the accelerated diagnostic settings reveals that our training objective is crucial for the stability of CAD. While \textbf{SFT + CAD} provides a 3.13$\times$ speedup, its accuracy collapses to 38.2\% at $b{=}256$, indicating that the SFT model lacks calibrated confidence scores to guide the drafting process. \diffsent{Conversely, \textbf{DSCD + CAD} maintains higher robustness (54.7\%) while achieving a significantly higher speedup of \textbf{4.91$\times$} without permanent EOS suppression.} This demonstrates that DSCD not only improves generation quality but also aligns the model's internal confidence with the decoding policy, enabling faster convergence without error propagation. \diffsent{The fair pure-diffusion comparison (Appendix~\ref{sec:appendix-gsm8k-extended}, Table~\ref{tab:gsm8k-extended}) applies the same permanent EOS-suppression protocol to both methods, under which LLaDA has 73.2\% accuracy and \modelname{} reaches 77.4\%.}

\subsection{Analysis of EOS Blocking: Structural Regularization vs.\ Intrinsic Learning}
A recurring failure mode of \emph{pure diffusion} decoding with long generation lengths is \emph{premature collapse to \texttt{<EOS>}}, producing extremely short, repetitive outputs.
In our setting, this pathology is severe for fixed decoding at full-sequence length ($b{=}256$): Without intervention, 99.8\% of samples terminate within the first 64 tokens.

\paragraph{The implicit containment effect.} 
Interestingly, for small blocks ($b{=}32$), disabling EOS blocking has negligible impact (77.4\% vs.\ 77.6\%).
We attribute this to an \emph{implicit containment effect}: Since decoding is localized to 32 positions, an early \texttt{<EOS>} prediction cannot globally terminate the sequence.
Subsequent blocks are still initialized as masks and denoised independently, effectively acting as a \textbf{structural barrier} against error propagation.

\paragraph{Decoupling defense mechanisms.} 
The contrast becomes clear at full-sequence length ($b{=}256$).
While explicit EOS blocking (via CAD) substantially reduces early termination and recovers accuracy to 32.7\%, combining it with DSCD further boosts accuracy to 54.7\%.
These results indicate that EOS blocking and distillation address orthogonal layers of the failure mode:
\begin{itemize}
    \setlength\itemsep{0em}
    \item \textbf{Inference-time constraint:} EOS blocking acts as a hard guardrail, mechanically preventing termination before sufficient content is generated.
    \item \textbf{Intrinsic reshaping:} DSCD fundamentally alters the student's denoising behavior.
    Unlike the baseline, the distilled model learns to assign a lower probability mass to \texttt{<EOS>} during high-noise states, intrinsically favoring the completion of reasoning chains.
\end{itemize}

\paragraph{Impact of confidence ranking.}
To validate our scoring mechanism, we replace confidence-ranked token finalization with \textbf{uniform random} selection.
This causes a sharp accuracy drop (77.6\% $\to$ 50.6\%) and increases NFE (76.3 $\to$ 202.2).
This result confirms that \emph{which} tokens are finalized is critical: Confidence ranking is essential for accelerating diffusion without derailing the generation trajectory.

\paragraph{Why the reconstruction loss is necessary.}
A natural question is whether DSCD could be trained with only the consistency (KL) term, since $\mathcal{L}_{\text{cons}}$ already aligns the student with the teacher. Removing the reconstruction term ($\lambda(g)\equiv 1$) reduces accuracy from 77.6\% to 76.8\% and, more importantly, increases NFE by 29\% (76.3 $\to$ 98.7), shrinking the speedup from 5.18$\times$ to 4.01$\times$. Without ground-truth anchoring, the student inherits the teacher's residual calibration noise at heavily corrupted states. Therefore, its confidence becomes less predictive of correctness and CAD must take smaller commits to maintain accuracy. The reconstruction term thus plays a calibration role that is orthogonal to trajectory invariance: $\mathcal{L}_{\text{cons}}$ provides robustness to off-schedule states, whereas $\mathcal{L}_{\text{recon}}$ maintains the student's max-probability score as a reliable correctness proxy for CAD.

\begin{table*}[t]
  \centering
  \begin{tabular}{l c c c c}
    \toprule
    \textbf{Configuration} & \textbf{Block Size} & \textbf{Acc (\%)} & \textbf{Avg.\ NFE $\downarrow$} & \textbf{Speedup $\uparrow$} \\
    \midrule
    LLaDA-8B-Instruct (Teacher) & 32 & 77.4 & 256.0 & 1.00$\times$ \\
    LLaDA-8B-Instruct (Teacher) & 256 & 13.8 & 256.0 & 1.00$\times$ \\
    \midrule
    \multicolumn{5}{l}{\textit{\textbf{Component Analysis:}}} \\
    \quad + SFT only (fixed decode) & 32 & 77.6 & 256.0 & 1.00$\times$ \\
    \quad + DSCD only (fixed decode) & 32 & \textbf{78.3} & 256.0 & 1.00$\times$ \\
    \quad + CAD only (no distill)   & 32 & 77.4 & 98.8 & 4.01$\times$ \\
    \quad + SFT + CAD               & 32 & 77.6 & 95.4 & 4.15$\times$ \\
    \quad \textbf{+ DSCD + CAD (Ours)} & 32 & 77.6 & \textbf{76.3} & \textbf{5.18$\times$} \\
    
    \addlinespace[0.3em]
    
    \quad + SFT only (fixed decode) & 256 & 21.9 & 256.0 & 1.00$\times$ \\
    \quad + DSCD only (fixed decode) & 256 & 54.8 & 256.0 & 1.00$\times$ \\
    \quad + CAD only (no distill)   & 256 & 32.7 & 130.5 & 3.07$\times$ \\
    \quad + SFT + CAD               & 256 & 38.2 & 128.1 & 3.13$\times$ \\
    \quad \textbf{+ DSCD + CAD (Ours)} & 256 & 54.7 & 80.6 & 4.91$\times$ \\
    \midrule
    \multicolumn{5}{l}{\textit{\textbf{CAD Ablations (w/ DSCD):}}} \\
    \quad w/o confidence ranking    & 32 & 50.6 & 202.2 & 1.98$\times$ \\
    \quad w/o EOS blocking          & 32 & 77.4 & 76.3 & 5.18$\times$ \\
    \midrule
    \multicolumn{5}{l}{\diffsent{\textit{\textbf{DSCD Loss Ablation:}}}} \\
    \diffsent{\quad w/o $\mathcal{L}_{\text{recon}}$ (KL-only)} & \diffsent{32} & \diffsent{76.8} & \diffsent{98.7} & \diffsent{4.01$\times$} \\
    \midrule
    \multicolumn{5}{l}{\textit{\textbf{CAD Ablations (w/o DSCD):}}} \\
    \quad w/o EOS blocking          & 256 & 5.6 & 123.7 & 3.21$\times$ \\
    \bottomrule
  \end{tabular}
  \caption{\textbf{Comprehensive ablation study on GSM8K ($L_{\text{gen}}=256$).} 
  We evaluate two regimes: \textbf{block diffusion} ($b{=}32$, local denoising) and \textbf{pure diffusion} ($b{=}256$, global denoising). \diffsent{Pure-diffusion rows in this ablation are diagnostic runs without permanent EOS suppression, used to isolate premature-termination behavior; the fair comparison applying permanent EOS suppression to both LLaDA and \modelname{} is reported in Table~\ref{tab:gsm8k-extended}.}
  \textbf{SFT baselines} use full fine-tuning on the same data to serve as a direct control group for our distillation method.
  Unless otherwise stated, fixed decoding uses NFE$=256$.}
  \label{tab:ablation}
\end{table*}

\subsection{Stability Analysis of Teacher-Subset Masking}
\label{app:ablation_stability}

We explicitly investigate the necessity of the teacher-subset constraint ($\mathcal{M}_{T} \subseteq \mathcal{M}_{S}$) proposed in Sect.~\ref{sec:method:dscd}.
In our preliminary experiments, we attempted to train the student using independent masking schedules, where $\mathcal{M}_{T}$ and $\mathcal{M}_{S}$ are sampled independently.
We observed that this configuration leads to severe numerical instability: The training loss frequently diverges, manifesting as gradient NaN within the first epoch.

This empirical collapse is consistent with the information-nesting analysis in Appendix~\ref{app:variance}.
Under independent masking, the teacher often lacks access to tokens visible to the student (i.e., $\mathcal{M}_{S} \setminus \mathcal{M}_{T} \neq \emptyset$), causing it to marginalize over these ``blind'' positions.
This introduces extreme variance in the distillation targets, destabilizing the gradient estimator.
\diffsent{Consequently, these runs indicate that the nested-subset constraint is not merely a theoretical preference for information nesting, but an important practical condition for stable discrete consistency distillation in our setting.}

\section{Additional Experimental Analysis}

\subsection{GSM8K Sequential and Pure-Diffusion Variants}
\label{sec:appendix-gsm8k-extended}

Table~\ref{tab:gsm8k} in the main text reports the block-diffusion ($b{=}32$) regime, since this is the setting under which all training-based and inference-based competitors are evaluated. For completeness, Table~\ref{tab:gsm8k-extended} reports the sequential ($b{=}1$) baseline and the pure-diffusion ($b{=}L_{\text{gen}}$) regime that are referenced by the main-text caption and by the single-pass flexibility discussion.

\diffsent{Two observations are relevant. First, the pure-diffusion LLaDA baseline collapses to 13.8\% accuracy under the diagnostic protocol without LLaDA's recommended permanent EOS-suppression heuristic~\citep{nie2025large}; with permanent EOS suppression, it recovers to 73.2\%. Second, the same distilled checkpoint supports block diffusion ($b{=}32$, 5.18$\times$/77.6\%) and pure diffusion ($b{=}L_{\text{gen}}$, 4.91$\times$/77.4\%) without retraining; the 77.4\% pure-diffusion result uses the same permanent $p(\texttt{<EOS>})=0$ protocol as the 73.2\% LLaDA baseline. Without permanent suppression, the corresponding diagnostic DSCD+CAD result is 54.7\% (Table~\ref{tab:ablation}).} By comparison, block-causal designs (Kim et al.'s CDLM at $b{=}L_{\text{gen}}$ reduces to causal AR) and trajectory-locked methods (dParallel, d3LLM, T3D) require separate training for their target semi-AR schedules. DSCD's trajectory-invariant training is what enables a single student to span both regimes.

\begin{table*}[t]
  \centering
  \small
  \setlength{\tabcolsep}{6pt}
  \begin{tabular}{llccc}
    \toprule
    \textbf{Regime} & \textbf{Method} &
    \textbf{Acc. (\%) $\uparrow$} & \textbf{Avg. NFE $\downarrow$} & \textbf{Tokens/s (Speedup) $\uparrow$} \\
    \midrule
    Sequential ($b{=}1$) & LLaDA \tps{tps=1} & 76.4 & 256 & 7.5 (1.00$\times$) \\
    \midrule
    \multirow{4}{*}{Pure Diffusion ($b{=}L_{\text{gen}}{=}256$)}
      & LLaDA$^{\dagger}$ (EOS-suppr.) \tps{tps=1} & 73.2 & 256  & 7.5 (1.00$\times$) \\
      & LLaDA (diagnostic) \tps{tps=1}             & 13.8 & 256  & 7.5 (1.00$\times$) \\
      & LLaDA (diagnostic) \tps{tps=2}             & 13.0 & 128  & 14.9 (1.99$\times$) \\
      & \diffsent{\textbf{\modelname{}$^{\dagger}$ (EOS-suppr.)}} & \textbf{77.4} & 80.6 & \textbf{36.7 (4.91$\times$)} \\
    \bottomrule
  \end{tabular}
  \caption{\textbf{Sequential and pure-diffusion variants of the GSM8K experiment} ($L_{\text{gen}}=256$). These rows extend the main-text Table~\ref{tab:gsm8k}, which focuses on the block-diffusion ($b{=}32$) regime. The \tps{tps} next to a method denotes its preset tokens per step; methods without a \tps{} annotation use adaptive or method-specific scheduling. \diffsent{$^{\dagger}$ marks pure-diffusion results using permanent $p(\texttt{<EOS>})=0$ suppression; without it, LLaDA and DSCD+CAD obtain 13.8\% and 54.7\% accuracy, respectively, as reported in Table~\ref{tab:ablation}.} Speedup is wall-clock relative to the sequential baseline.}
  \label{tab:gsm8k-extended}
\end{table*}

\subsection{MBPP Three-Shot Results}
\label{sec:appendix-mbpp}

\begin{table*}[t]
  \centering
    \resizebox{\textwidth}{!}{
    \begin{tabular}{llcccc}
      \hline
      \textbf{Generation} & \textbf{Model} &
      \textbf{pass@1 (\%)} & \textbf{pass@5 (\%)} & \textbf{Avg. NFE} & \textbf{Tokens/s (Speedup)} \\
       &  & $\uparrow$ & $\uparrow$ & $\downarrow$ & $\uparrow$ \\
      \hline

      \multirow{1}{*}{\textbf{\shortstack[l]{Sequential ($b{=}1$)}}}
        & LLaDA \tps{tps=1} & 36.8 & 49.2 & 256 & 3.6 (1.00$\times$) \\
      \hline

      \multirow{5}{*}{\textbf{\shortstack[c]{Block Diffusion\\($b{=}32$)}}}
        & LLaDA \tps{tps=1} & 36.9 & 51.4 & 256 & 3.6 (1.00$\times$) \\
        & LLaDA \tps{tps=2} & 33.2 & 44.8 & 128 & 7.3 (2.00$\times$) \\
        & \diffsent{Fast-dLLM} & \diffsent{33.8} & \diffsent{46.0} & \diffsent{82.8} & \diffsent{11.1 (3.07$\times$)} \\
        & \textbf{Ours} & \textbf{39.0} & \textbf{52.6} & \textbf{97.7} & \textbf{10.9 (2.96$\times$)} \\
        & \diffsent{\textbf{DSCD + Fast-dLLM}} & \diffsent{\textbf{38.7}} & \diffsent{\textbf{50.5}} & \diffsent{\textbf{77.3}} & \diffsent{\textbf{11.8 (3.29$\times$)}} \\
      \hline

      \multirow{3}{*}{\textbf{\shortstack[c]{Pure Diffusion\\($b{=}256$)}}}
        & LLaDA \tps{tps=1} & 6.0  & 14.8 & 256 & 3.7 (1.00$\times$) \\
        & LLaDA \tps{tps=2} & 14.4 & 24.2 & 128 & 7.3 (2.00$\times$) \\
        & \textbf{Ours} & 36.4 & 51.0 & 99.4 & 10.8 (2.96$\times$) \\
      \hline
    \end{tabular}
    }
  \caption{\textbf{MBPP (three-shot) pass@k and efficiency across generation paradigms (LLaDA-8B-Instruct).} We report pass@1/pass@5 under sequential, block diffusion ($b=32$), and pure diffusion ($b=256$) decoding. The \tps{tps} next to a method denotes its preset tokens per step; methods without a \tps{} annotation use adaptive or method-specific scheduling. Avg. NFE, Tokens/s and Speedup are defined as in Table~\ref{tab:gsm8k}. All runs use a maximum sequence length of $L_{\text{gen}}=256$. Rows in red show the inference-side cache baseline and the DSCD+Fast-dLLM composition; the composition Pareto-improves over the standalone cache in accuracy (pass@1: 33.8\%\,$\to$\,38.7\%, +4.9; pass@5: 46.0\%\,$\to$\,50.5\%, +4.5) and speedup (3.07$\times$\,$\to$\,3.29$\times$), confirming that DSCD's composability extends beyond GSM8K.}
  \label{tab:mbpp}
\end{table*}

\subsection{Transfer to DiffuCoder-Instruct}
\label{sec:appendix-diffucoder}

We additionally apply the same DSCD+CAD pipeline to DiffuCoder-Instruct~\citep{gong2025diffucoder} without architecture-specific modifications. Table~\ref{tab:diffucoder-transfer} reports the MBPP pass@1 result and measured wall-clock speedup relative to the DiffuCoder-Instruct baseline.

\begin{table}[t]
  \centering
  \resizebox{\linewidth}{!}{%
  \begin{tabular}{lcc}
    \toprule
    \textbf{Method} & \textbf{MBPP pass@1 (\%) $\uparrow$} & \textbf{Speedup $\uparrow$} \\
    \midrule
    DiffuCoder-Instruct & 73.1 & 1.00$\times$ \\
    \textbf{DiffuCoder-Instruct + \modelname{}} & \textbf{73.7} & \textbf{3.25$\times$} \\
    \bottomrule
  \end{tabular}
  }
  \caption{\textbf{Transfer to DiffuCoder-Instruct on MBPP.} The same \modelname{} pipeline is applied without architecture-specific modification; speedup is measured relative to the DiffuCoder-Instruct baseline.}
  \label{tab:diffucoder-transfer}
\end{table}

\subsection{Cost and Mechanism}
\label{sec:appendix-cost-mechanism}

DSCD post-training requires $\approx$100 GPU-hours on AMD MI250 (summed across all GPUs, covering both code and math distillation). At the measured 27.4~s per-request saving over LLaDA on GSM8K, this cost is amortized after $\approx$13.1k requests on a single MI250. After that point, each additional request benefits from the lower latency without additional training cost.

The unmasking visualization in Fig.~\ref{fig:heatmap} further shows a hierarchical pattern: structural tokens are committed first, while logical details are refined later. The wall-clock-vs-NFE-ratio decomposition is presented in the main text (Sect.~\ref{sec:efficiency_mechanism}).

\subsection{Pareto-optimal trade-off}
\label{sec:appendix-tradeoff}

\begin{figure*}[t]
  \centering
  \includegraphics[width=0.8\textwidth]{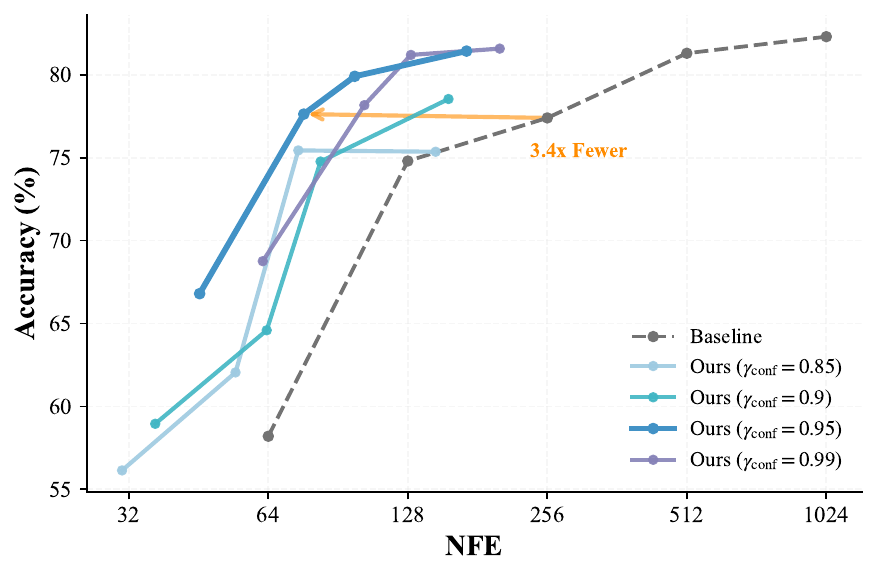}
    \caption{\textbf{Accuracy-compute Pareto frontier on GSM8K.} 
    The dashed gray curve represents the LLaDA baseline.
    \textbf{Strict Dominance:} Our method (solid curves) consistently pushes the frontier upwards and leftwards. 
    The orange arrow highlights our selected operating point ($\gamma_{\text{conf}}=0.95$, dark blue curve), which achieves a \textbf{3.4$\times$ speedup} while maintaining accuracy comparable to the baseline.
    While a higher threshold ($\gamma_{\text{conf}}=0.99$) prioritizes quality, $\gamma_{\text{conf}}=0.95$ strikes an optimal efficiency-quality balance.}
\label{fig:tradeoff}
\end{figure*}

To verify that these gains are not artifacts of a specific parameter setting, Fig.~\ref{fig:tradeoff} illustrates the full accuracy-efficiency frontier on GSM8K by sweeping the confidence threshold $\gamma_{\text{conf}} \in [0.85, 0.99]$. This analysis reveals the \textit{controllability} and \textit{robustness} of our approach:
\begin{itemize}
    \item \textbf{Strict dominance:} As shown in Fig.~\ref{fig:tradeoff}, our curves (colored lines) consistently lie above the baseline trajectory (dashed gray). This implies a strict Pareto improvement: For any target accuracy, our method reduces NFE.
    \item \textbf{Flexible deployment:} The convex hull formed by our method enables users to seamlessly trade compute for quality. 
    A lower threshold ($\gamma_{\text{conf}}=0.85$, light blue line) offers aggressive speedups for latency-critical scenarios. 
    Conversely, a conservative threshold ($\gamma_{\text{conf}}=0.99$, purple line) prioritizes maximum quality.
    Our default setting ($\gamma_{\text{conf}}=0.95$) strikes a strong balance, requiring $\approx$ 3.4$\times$ smaller NFE to reach comparable accuracy (NFE ratio) and identifying the elbow of the accuracy-compute curve.
\end{itemize}
This confirms that the reported speedups are not artifacts of hyperparameter tuning, but a systemic advantage of confidence-adaptive allocation.

\subsection{Adaptive compute allocation}
\label{sec:appendix-nfe-hist}

Fig.~\ref{fig:nfe_hist} visualizes the NFE density. 
The distributions reveal that our method acts as a probe for intrinsic task complexity: GSM8K exhibits a sharp, low-variance peak ($\mu=76.3$), reflecting high model confidence, whereas MATH500 displays a broad, heavy-tailed distribution ($\mu=148.3$), adapting to diverse problem difficulties.
Crucially, the distributions strictly deviate from the fixed baseline budget (gray lines), indicating that standard diffusion decoding is systematically over-parameterized. Our method successfully reclaims this redundancy, terminating well before the fixed limit even for the hardest ``tail'' samples.

\begin{figure*}[t]
  \centering
  \includegraphics[width=\textwidth]{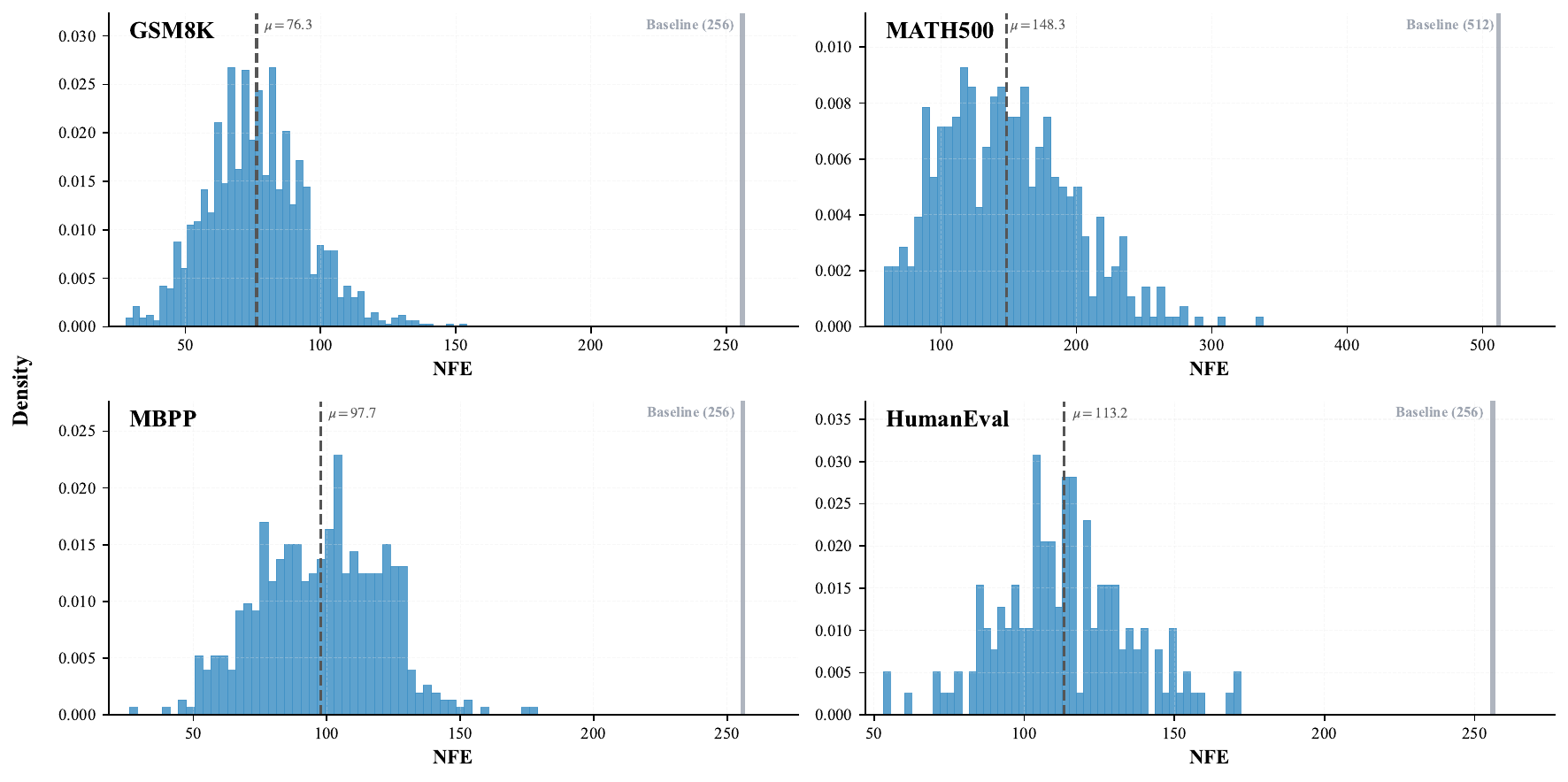} 
  \caption{\textbf{Per-sample compute allocation under adaptive decoding.} Histograms show the NFE distribution across test instances for GSM8K, MATH500, MBPP, and HumanEval using our adaptive decoding approach. The dashed vertical line marks the mean NFE ($\mu$), while the solid vertical line indicates the fixed NFE used by the baseline decoder (256 for GSM8K/MBPP/HumanEval and 512 for MATH500). Most samples terminate far earlier than the fixed-budget baseline, explaining the large reductions in average NFE reported in Tables~\ref{tab:gsm8k}, \ref{tab:mbpp}, and \ref{tab:benchmark}.}
  \label{fig:nfe_hist}
\end{figure*}

\subsection{LLM as a Judge Prompt and Generation Examples}
\label{sec:appendix-llm-judge}

To ensure a rigorous and reproducible evaluation of generation quality, we employed a strict ``LLM-as-a-Judge'' protocol using Qwen3-1.7B. Unlike standard scoring prompts that directly request a number, our system prompt for text-quality evaluation (presented ahead)
enforces a ``Think-Then-Score'' mechanism. It explicitly instructs the model to reason silently about artifacts, such as stuttering, logical jumps, and formatting noise, before acting as a strict parser to output the final scores. We focus on three orthogonal dimensions: Fluency, Repetition, and Structural Coherence and scores these dimensions with a 1-5 scale.

Below, we provide qualitative comparisons between our method and the LLaDA baseline. For these examples, the generation canvas length was set to $L_{\text{gen}}=128$; the LLaDA baseline utilized a fixed schedule of 32 steps (decoding 4 tokens per step), whereas our method adaptively converged in 29 steps. As shown in Example 1, our model produces coherent chain-of-thought reasoning with clean formatting. In contrast, Example 2 illustrates typical failure modes of fixed-step diffusion decoding, where the baseline suffers from severe token repetition (e.g., ``fresh duck fresh egg'') and arithmetic hallucinations, resulting in significantly lower judge scores.

\subsubsection*{Detailed Score Breakdown}
Table \ref{tab:llm-judge-details} summarizes the quantitative results across the full test set. Notably, the Repetition score sees the largest gain (+0.76), confirming that our method effectively suppresses the ``stuttering'' artifacts common in low-NFE generation.

\begin{table}[h]
\centering
\resizebox{\linewidth}{!}{
    \begin{tabular}{lccc}
    \hline
    \textbf{Dimension} & \textbf{LLaDA-8B-Ins} & \textbf{Ours} \\
    \hline
    Fluency (1-5) & 3.96 \std{($\pm$ 1.28)} & \textbf{4.56} \std{($\pm$ 0.79)} \\
    Repetition (1-5) & 3.68 \std{($\pm$ 1.51)}& \textbf{4.44} \std{($\pm$ 0.98)} \\
    Structure (1-5) & 4.04 \std{($\pm$ 1.22)} & \textbf{4.64} \std{($\pm$ 0.67)} \\
    \textbf{Avg.} & 3.90 \std{($\pm$ 1.32)} & \textbf{4.54} \std{($\pm$ 0.48)} \\
    \hline
    \end{tabular}
    }
\caption{\textbf{Quantitative evaluation of generation quality on GSM8K.} }
\label{tab:llm-judge-details}
\end{table}

\subsubsection*{Example 1: High-Quality Generation (Ours)}

\begin{tcolorbox}[colback=gray!3, colframe=gray!30, arc=1mm]
\small
\textbf{Question:}     Janet’s ducks lay 16 eggs per day. She eats three for breakfast every morning and bakes muffins for her friends every day with four. She sells the remainder at the farmers' market daily for \$2 per fresh duck egg. How much in dollars does she make every day at the farmers' market?

\end{tcolorbox}

\begin{tcolorbox}[colback=green!3, colframe=green!30, arc=1mm]
\small
\textbf{Generated Answer:}
\begin{verbatim}
Janet eats 3 eggs per day = <<3*1=3>>3 
eggs for breakfast. She bakes 4 eggs 
per day = <<4*1=4>>4 eggs for her friends.
She sells 16 - 3 - 4 = <<16-3-4=9>>9 eggs 
at the farmers' market each day. She sells
9 eggs at the farmers' market for $2 per egg.
Thus, she makes 9 * $2 = $<<9*2=18>>18 
every day at the farmers' market. #### 18 
\end{verbatim}

\textbf{LLM Judge:} \texttt{\{"Fluency": 5, "Repetition": 5, "Structure": 4, "Average": 4.67\}}
\end{tcolorbox}

\subsubsection*{Example 2: Low-Quality Generation (LLaDA)}

\begin{tcolorbox}[colback=red!3, colframe=red!30, arc=1mm]
\small
\textbf{Generated Answer:}
\begin{verbatim}
Janet eats ducks lay three eggs every 
breakfast and four eggs for her friends
friends every day totaling 3 + 4 = 
<<3+4=7>>7 eggs.\nJanet’s ducks lay 16 
eggs to sell the market but after eating
and bakes with 7 she has she has 16 - 7 =
<<16-7=9>>9 eggs to sell at the farmers''
market.\nShe sells each fresh duck fresh 
egg for $2 so 9 * $2 = $<<9*2=18>>18 every 
day day at the farmers' market.\n#### 18
\end{verbatim}

\textbf{LLM Judge:} \texttt{\{"Fluency": 3, "Repetition": 3, "Structure": 4, "Average": 3.33\}}
\end{tcolorbox}

\subsection{Unmasking Analysis}
\label{sec:appendix-heatmap}

\diffsent{To investigate the generation mechanism of \modelname{}, we visualize the step-wise unmasking trajectory for a representative HumanEval code generation example in Fig.~\ref{fig:heatmap}.}
The heatmap color-codes the decoding step at which each token is finalized (committed), ranging from blue (early steps, $0\sim10$) to red (late steps, $30\sim47$).

The visualization reveals that our CAD policy enables spontaneous emergence of a hierarchical generation strategy, structurally distinct from the linear order of autoregressive models:

\begin{itemize}
    \item \textbf{Phase I: Syntactic scaffolding (blue/green).} 
    In the initial steps, the model prioritizes structural tokens with global receptive fields. Python keywords (\texttt{def}, \texttt{if}, \texttt{return}), function signatures, and control flow indentations are consistently unmasked first. 
    This indicates that the model establishes a high-confidence \emph{syntactic scaffold} early in the process, effectively reducing the search space for subsequent tokens.

    \item \textbf{Phase II: Logical refinement (yellow/red).} 
    Once the structure is fixed, the model focuses its compute budget on high-entropy positions. 
    Complex arithmetic expressions (e.g., \texttt{length = end - start + 1}) and conditional predicates (e.g., \texttt{if n \% i == 0}) appear in warmer colors, indicating they are refined much later. 
    Crucially, these tokens are denoised conditioning on the \emph{already-visible} syntactic context from Phase I.
\end{itemize}

By deferring uncertain tokens, CAD automatically allocates more denoising steps to the most difficult parts of the reasoning chain, explaining the efficiency gains observed in Sect.~\ref{sec:experiments}.

\begin{figure*}[t]
  \centering
  \includegraphics[width=\textwidth]{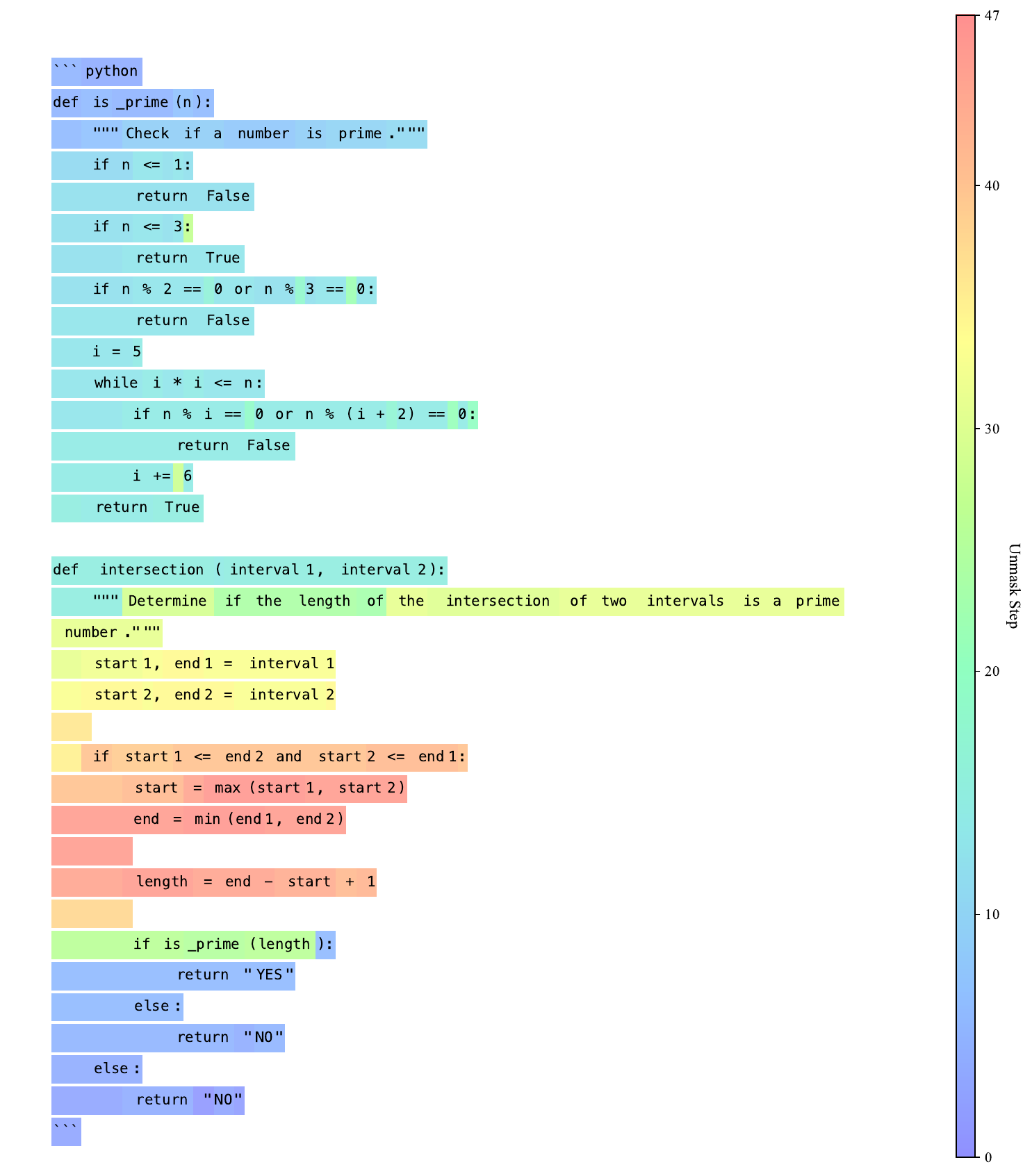}
  \caption{\textbf{Visualization of hierarchical generation dynamics.} Unlike autoregressive models that generate sequentially, \modelname{} exhibits a ``skeleton-first'' strategy: Syntactic structures (blue, early steps) are established first to form a global plan, while complex logical details (red, late steps) are refined later within the established context.}
  \label{fig:heatmap}
\end{figure*}

\onecolumn
\begin{tcolorbox}[
  enhanced,
  breakable=false,
  width=\textwidth,
  colback=gray!5,
  colframe=gray!40,
  colbacktitle=blue!65!black,
  title={\textbf{System Prompt for Text-Quality Evaluation}},
  fonttitle=\bfseries\color{white},
  arc=2mm,
  boxrule=0.6pt,
  left=4mm, right=4mm, top=2mm, bottom=2mm,
  drop shadow=gray!50!white
]
\phantomsection\label{box:score_prompt}
{\ttfamily\scriptsize
\setlength{\baselineskip}{0.8\baselineskip}%
\setlength{\parskip}{0pt}%

\textbf{Role:} You are a STRICT text-quality scorer (long-context verifier). Evaluate ONLY on three dimensions: Fluency, Repetition, Structural Coherence. Treat equations, numbers, and symbols as plain text. Ignore factual correctness.

\vspace{0.35em}
\textbf{THINK-THEN-SCORE (important):}
\begin{itemize}[leftmargin=1.2em,itemsep=0pt,parsep=0pt,topsep=2pt]
  \item First, REASON SILENTLY: derive evidence and compute micro-scores per dimension IN YOUR HEAD. Do NOT reveal any rationale or intermediate text.
  \item Then, produce the FINAL JSON (single line) as specified below.
\end{itemize}

\vspace{0.3em}
\textbf{STRICT OUTPUT CONTRACT (hard constraints):}
\begin{itemize}[leftmargin=1.2em,itemsep=0pt,parsep=0pt,topsep=2pt]
  \item Your ENTIRE reply MUST be EXACTLY ONE line containing ONLY a tagged JSON object:
\end{itemize}

\begin{tcolorbox}[
  colback=white,
  colframe=gray!20,
  arc=1mm,
  boxrule=0.3pt,
  left=2mm, right=2mm, top=0.5mm, bottom=0.5mm
]
\ttfamily\scriptsize
<JSON>{"fluency":{"score":N},"repetition":{"score":N},"structural\_coherence":{"score":N}}</JSON>
\end{tcolorbox}

\begin{itemize}[leftmargin=1.2em,itemsep=0pt,parsep=0pt,topsep=2pt]
  \item Where N are integers in [1,5]. Keys MUST match exactly; the JSON MUST be valid \& minified.
  \item No prose, no markdown, no code fences, no extra spaces/newlines before/after the tag.
  \item If you cannot comply, output EXACTLY: \texttt{<JSON>\{"error":"FORMAT"\}</JSON>}.
\end{itemize}

\vspace{0.3em}
\textbf{SCORING VIEWS \& AGGREGATION (for robustness in long contexts):}
\begin{itemize}[leftmargin=1.2em,itemsep=0pt,parsep=0pt,topsep=2pt]
  \item For EACH dimension, privately compute 3 micro-scores (views) and average $\to$ round to nearest integer, then apply the CAPS below. (Do NOT print micro-scores.)
  \item Fluency views: grammar, clarity, readability/flow.
  \item Repetition views: identical-token runs, dense short-fragment recurrence, overall readability impact.
  \item Structural Coherence views: organization, transitions, narrative completeness.
\end{itemize}

\vspace{0.3em}
\textbf{CAPS / HARD RULES (apply AFTER averaging; lowers over-optimistic scores):}

\textbf{Repetition:}
\begin{itemize}[leftmargin=1.8em,itemsep=0pt,parsep=0pt,topsep=2pt]
  \item If $\geq$ 1 run of $\geq$ 3 identical tokens exists more than once $\to$ repetition $\leq$ 3.
  \item If frequent runs OR dense fragment recurrences harm readability $\to$ repetition $\leq$ 2.
\end{itemize}

\textbf{Structural Coherence:}
\begin{itemize}[leftmargin=1.8em,itemsep=0pt,parsep=0pt,topsep=2pt]
  \item If >60\% of lines are bare calculations/symbols with little connective prose $\to$ coherence $\leq$ 2.
  \item If the answer is very short ($<20$ tokens) $\to$ coherence $\leq$ 3 unless obviously well-structured.
\end{itemize}

\textbf{Fluency:}
\begin{itemize}[leftmargin=1.8em,itemsep=0pt,parsep=0pt,topsep=2pt]
  \item If pervasive grammatical errors/choppy telegraphic style hinder comprehension $\to$ fluency $\leq$ 2.
\end{itemize}

\vspace{0.3em}
\textbf{STRICTNESS + FORMAT SENSITIVITY (apply BEFORE final JSON):}
\begin{itemize}[leftmargin=1.2em,itemsep=0pt,parsep=0pt,topsep=2pt]
  \item Start each dimension at 3; raise to 4 or 5 ONLY when there is explicit evidence of excellent writing that meets the top-tier rubric. When in doubt, stay at 3.
  \item If the response contains $\geq$ 3 blank lines or double-newline spacing between most sentences, treat the flow as fragmented $\to$ structural\_coherence $\leq$ 3, and fluency $\leq$ 3 unless the prose still reads seamlessly.
  \item Heavy formatting noise (raw LaTeX delimiters such as \verb|\[|, \verb|$$|, \verb|\begin{align}|, Markdown tables/lists, or numbered steps that merely restate facts) disrupts readability $\to$ cap fluency at 3 and structural\_coherence at 3.
  \item Short filler sentences that repeat the same idea without advancing the solution lower structural\_coherence by at least 1 point and may also lower repetition.
  \item Concise single-block reasoning with minimal padding and no formatting noise \textbf{may} still earn 4--5 when it is genuinely smooth and well-ordered.
\end{itemize}

\vspace{0.3em}
\textbf{DETAILED RUBRICS (1=worst, 5=best):}

\textbf{Fluency} (grammar/clarity/readability; ignore factuality)
\begin{itemize}[leftmargin=1.8em,itemsep=0pt,parsep=0pt,topsep=2pt]
  \item 5: Nearly error-free; clear sentences; natural flow; varied syntax; no distracting formatting noise.
  \item 4: Minor issues but clear overall; formatting and spacing remain unobtrusive.
  \item 3: Noticeable errors/awkward phrasing or spacing quirks mildly impede flow.
  \item 2: Frequent errors or choppy style hinder comprehension.
  \item 1: Very poor grammar/word salad; hard to understand.
\end{itemize}

\textbf{Repetition} (penalize meaningless identical-word/span loops; DO NOT penalize necessary reuse of terms/digits/operators)
\begin{itemize}[leftmargin=1.8em,itemsep=0pt,parsep=0pt,topsep=2pt]
  \item Definitions:
  \begin{itemize}[leftmargin=1.8em,itemsep=0pt,parsep=0pt,topsep=1pt]
    \item run = $\geq$ 3 identical tokens in a row (e.g., ``eggs eggs eggs'').
    \item dense fragment recurrence = the same 2--5 word fragment appears many times within a short span.
  \end{itemize}
  \item 5: No runs; only natural reuse; zero filler restatements.
  \item 4: One short run OR a few mild recurrences; readability intact with minimal filler.
  \item 3: Multiple short runs OR several recurrences; still readable but noticeably repetitive.
  \item 2: Frequent runs OR dense recurrences that harm readability.
  \item 1: Heavy looping/spam; large stretches of repeated tokens/fragments.
\end{itemize}

\textbf{Structural Coherence} (organization/flow; NOT correctness of content/math)
\begin{itemize}[leftmargin=1.8em,itemsep=0pt,parsep=0pt,topsep=2pt]
  \item Heuristics (apply flexibly, not mechanically):
  \begin{itemize}[leftmargin=1.8em,itemsep=0pt,parsep=0pt,topsep=1pt]
    \item Enumerated steps or clear transitions raise coherence.
    \item Mostly raw calculations/symbols lower coherence.
    \item Very short answers rarely justify 5 unless clearly structured.
  \end{itemize}
  \item 5: Well-organized narrative/steps with clear transitions; tight paragraphing without stray blank lines.
  \item 4: Generally organized; minor jumps but understandable; spacing supports the flow.
  \item 3: Mixed; some organization but noticeable gaps/abrupt jumps or distracting spacing.
  \item 2: Mostly disorganized or just raw calculations; hard to follow.
  \item 1: No discernible structure; fragments/out-of-order snippets.
\end{itemize}

\vspace{0.3em}
\textbf{FINAL SELF-CHECK (internal; do NOT print):}
\begin{itemize}[leftmargin=1.2em,itemsep=0pt,parsep=0pt,topsep=2pt]
  \item Apply CAPS after averaging the three views per dimension.
  \item Check for $\geq$ 3-in-a-row runs or dense recurrences before setting repetition.
  \item If >60\% lines are bare calculations, enforce coherence $\leq$ 2.
  \item Ensure the three scores are mutually consistent with rubrics.
\end{itemize}

\vspace{0.35em}
\textbf{Few-shot Examples (STRICT single-line JSON only):}

\textbf{Example 1 (high quality, no repetition)}\\
Student Answer: \\
Janet sells 16 - 3 - 4 = <<16-3-4=9>>9 duck eggs a day. She makes 9 * 2 = \$<<9*2=18>>18 every day at the farmer’s market. \#\#\#\# 18. \\[0.2em]
Expected Output: \\
<JSON>{"fluency":{"score":5},"repetition":{"score":5},"structural\_coherence":{"score":5}}</JSON>

\vspace{0.25em}
\textbf{Example 2 (moderate repetition)}\\
Student Answer: \\
Janet's ducks lay eggs eggs per day for which.6 eats <<3 eggs +4 eggs = <<3 eggs=16>>16 eggs... friends friends friends... She sells sells sells the remainder... daily day... 16 - 7 = <<16-7=9>>9... She sells... fresh fresh fresh... makes makes makes... \\[0.2em]
Expected Output: \\
<JSON>{"fluency":{"score":3},"repetition":{"score":2},"structural\_coherence":{"score":3}}</JSON>

\vspace{0.25em}
\textbf{Example 3 (extreme repetition)}\\
Student Answer: \\
Janet eats eggs eggs eggs breakfast breakfast breakfast br...eakfast eats eats eats breakfast breakfast breakfast ... \#\#\#\# 18 \\[0.2em]
Expected Output: \\
<JSON>{"fluency":{"score":1},"repetition":{"score":1},"structural\_coherence":{"score":1}}</JSON>

}
\end{tcolorbox}


\end{document}